\documentclass{article}

\usepackage{arxiv}
\usepackage[utf8]{inputenc} 
\usepackage[T1]{fontenc}    
\usepackage{hyperref}       
\usepackage{url}            
\usepackage{booktabs}       
\usepackage{amsfonts}       
\usepackage{nicefrac}       
\usepackage{microtype}      
\usepackage{lipsum}		
\usepackage{graphicx}
\usepackage{natbib}
\usepackage{doi}

\usepackage[dvipsnames]{xcolor}

\newcommand{\norm}[1]{\left\lVert#1\right\rVert}
\newcommand{\snorm}[1]{\left\lvert#1\right\rvert}

\newcommand{\bs}[1]{\boldsymbol{#1}}

\newcommand{\s}{{\boldsymbol s}}
\newcommand{\n}{{\boldsymbol n}}
\newcommand{\x}{{\boldsymbol x}}
\newcommand{\y}{{\boldsymbol y}}
\newcommand{\z}{{\boldsymbol z}}
\newcommand{\w}{{\boldsymbol w}}

\newcommand{\I}{{\boldsymbol I}}
\newcommand{\A}{{\boldsymbol A}}

\newcommand{\Nc}{{\mathcal N}} 
\newcommand{\Lc}{{\mathcal L}} 

\newcommand{\Eb}{{\mathbb E}} 

\newcommand{\R}{\mathbb{R}}
\newcommand{\alphabar}{\bar{\alpha}}
\newcommand{\xzt}{\hat{\boldsymbol x }_{0|t}}

\newcommand{\gradxt}{\nabla_{\x_t}}
\newcommand{\argmin}[1]{\underset{#1}{\arg\min}}

\newcommand{\name}{SaFaRI} 

\usepackage{amsmath,amssymb,amsthm}
\usepackage{bm}
\usepackage{indentfirst}
\usepackage{booktabs}
\usepackage{multirow}
\usepackage{xcolor,colortbl}
\usepackage{color, colortbl}
\definecolor{Gray}{gray}{0.85}

\newtheorem{remark}{Remark} 

\newtheorem*{lemma1}{Lemma 1}

\usepackage{algpseudocode}
\usepackage{algorithm}

\usepackage{thmtools,thm-restate}

\title{Spatial-and-Frequency-aware Restoration method 
\\ for Images based on Diffusion Models}

\author{{Kyungsung Lee\thanks{Department of Mathematical Sciences,
Seoul National University, Seoul, South Korea}\;
\thanks{Equal contribution}} \\
	\texttt{kslee0304@snu.ac.kr} \\
	\And
	{Donggyu Lee\footnotemark[1]\;
    \footnotemark[2]} \\
	\texttt{dglee442@snu.ac.kr} \\
    \And
    {Myungjoo Kang\footnotemark[1]\;
    \thanks{Corresponding author}}  \\
	\texttt{mkang@snu.ac.kr} \\
}

\date{}

\renewcommand{\shorttitle}{Spatial-and-Frequency-aware Restoration method for Images based on Diffusion Models}

\hypersetup{
pdftitle={Spatial-and-Frequency-aware Restoration method for Images based on Diffusion Models},
pdfauthor={Kyungsung Lee, Donggyu Lee, Myungjoo Kang},
}

\begin{document}
\maketitle
\begin{abstract}
Diffusion models have recently emerged as a promising framework for Image Restoration (IR), owing to their ability to produce high-quality reconstructions and their compatibility with established methods. Existing methods for solving noisy inverse problems in IR, considers the pixel-wise data-fidelity. In this paper, we propose SaFaRI, a spatial-and-frequency-aware diffusion model for IR with Gaussian noise. Our model encourages images to preserve data-fidelity in both the spatial and frequency domains, resulting in enhanced reconstruction quality. We comprehensively evaluate the performance of our model on a variety of noisy inverse problems, including inpainting,  denoising, and super-resolution. Our thorough evaluation demonstrates that SaFaRI achieves state-of-the-art performance on both the ImageNet datasets and FFHQ datasets, outperforming existing zero-shot IR methods in terms of LPIPS and FID metrics.
\end{abstract}
\section{Introduction}
\label{sec:intro}

In the field of Image Restoration (IR), the overarching objective is to reconstruct an original image from a degraded or corrupted version of it. A classic approach \cite{rudin} is to use a variational model, which minimizes a cost function that includes a data-fidelity term and a regularization term \cite{datafidreg}. The data-fidelity term measures the difference between the restored image and the ground truth image, while the regularization term encourages the restored image to be smooth or have other desirable properties. If the measurement noise is Gaussian, the data-fidelity term is
\begin{equation} \label{eqn:datafid}
    \norm{\y-\A\x}_2^2\, ,
\end{equation} 
where $\y$ is a measurement, $\A$ is a degradation operator and $\x$ is a reconstructed image.

Meanwhile, diffusion models have gained widespread recognition as foundational models for generative modeling, offering a robust and flexible approach to data generation \cite{ddpm, scoresde, beatgan}. 

It exhibits remarkable capabilities in diverse image restoration tasks, encompassing deblurring, super-resolution, inpainting, and JPEG artifact removal. Notably, it achieves this through a zero-shot learning approach, effectively utilizing the generative priors embedded within a pre-trained model. \cite{ilvr, snips, ddrm, ddnm, dps, diffPIR, ccdf, mcg, pgdm, repaint, GDP, jpegddrm}. Numerous image restoration approaches employ conditional sampling to generate the desired restored image. This involves starting with pure noise and gradually denoising the image using diffusion models while concurrently pushing the data in the direction that minimizes data fidelity \cite{dps, pgdm, mcg}, hence minimizing (\ref{eqn:datafid}).

In this work, our primary objective extends beyond mere pixel-level data-fidelity to encompass perceptual data-fidelity. There have been numerous studies \cite{chen1994perceptual,jam2021r,lukac2017perceptual, wang2018perceptual} on image restoration that consider perceptual information. In IR, the improvement of perceptual information is essential for images, and hence, for performance \cite{yang2020fidelity, ma2022rectified,li2016perceptual}. Convolution operator and Fourier transform are two representative operators that handle perceptual information in images \cite{dpdmms, sfunet, fddgid, fuoli2021fourier}. These two operators are associated with spatial and frequency perceptual features, respectively \cite{ayyoubzadeh2021high, shao2023uncertainty}.

Drawing inspiration from these observations propose {\name}, a spatial-and-frequency-aware restoration method for images using diffusion model. {\name} is constructed by refining pixel-wise data fidelity using upsampling and Fourier Transform to enhance perceptual quality. Spatial-frequency-aware priors are integrated into the diffusion process in {\name}, allowing it to capture both low-level and high-level image features during restoration. We show that our method achieves significantly better performance than other zero-shot diffusion model-based image restoration methods.
\section{Background}
\label{sec:backgroud}

\begin{figure*}[t!]
  \centering
    \centerline{{\includegraphics[width=0.98\linewidth]{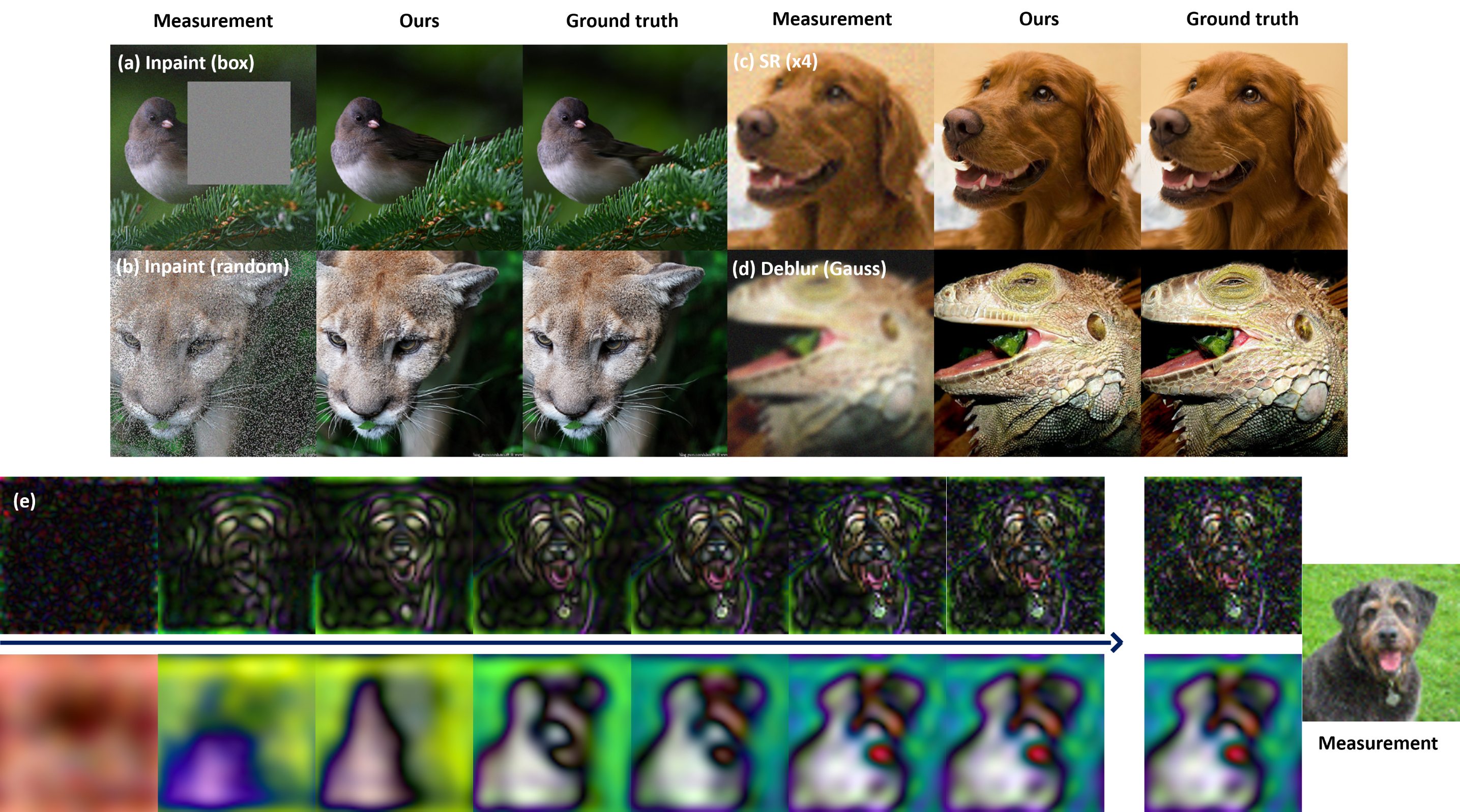}}}
  \caption{Examples and visual explanations of our method's functionality. (a)-(d): Results of the image restoration tasks: box-type inpainting, random-type inpainting, Gaussian deblurring and super resolution, respectively.
  (e): The first row illustrates the sequential changes in $A \xzt$ after applying high-pass filtering, leading to the final filtered image of $\y$, while the second row presents the low-pass counterparts.}
  \label{fig:intro_img}
\end{figure*}

\subsection{Score-based Diffusion Models}

The diffusion process represents the corruption of a clean data point by adding noise. This process is mathematically represented using SDE \cite{scoresde}, which captures the dynamic evolution of the data point as noise is gradually injected. The SDE governing the forward diffusion process is given by:
\begin{equation} \label{eqn:forwardsde}
    d\x = \bs{f}(\x,t)dt + g(t)d\w\, ,
\end{equation}
where $\bs{f}:\R^d\times[0,1]\rightarrow\R^d$ and $g:[0,1]\rightarrow\R$ are drift and diffusion coefficient respectively and $\w$ is the standard Wiener process.

The associated reverse SDE \cite{anderson} is expressed as: 
\begin{equation}
    d\x = \left[ \bs{f}(\x,t) -g^2(t)\nabla_{\x} \log p (\x) \right]dt + g(t)d\bar{\w}\, .
\end{equation}
Here, $dt$ is an infinitesimal negative time step and $\bar\w$ is the standard Wiener process running backward in time. The score function $\nabla_{\x} \log p_t (\x)$ can be approximated by a neural network $\s_{\theta}(\x,t)$ with score-matching \cite{ncsn, scoresde} objective:  
\begin{equation} \label{score_loss}
        \Eb_t \Big[ \lambda_t \Eb_{\x(0)}  \Eb_{\x(t)| \x(0)} \big[ 
                ||\ \s_{\theta}(\x(t), t) -\nabla_{\x(t)} \log p_{0t}(\x(t)|\x(0))\ ||_2^2\big] \Big],
\end{equation}

In this study, we follow DDPM \cite{ddpm}, which can be interpreted as a discretization of the VP-SDE: 
\begin{equation} \label{VPSDE}
    d\x =-\frac{\beta (t)}{2}\x dt + \sqrt{\beta (t)}d\w \, ,
\end{equation}
where $\beta (t)$ is a noise schedule. For the discretization step, we adhere to the notational conventions defined by DDPM \cite{ddpm}: $\x_i = \x(i/T)$, $\beta_i = \beta (i/T)$, $\alpha_i = 1-\beta_i$, and $\bar{\alpha}_i = \prod_{k=1}^i \alpha_i$ for $i=0,\, 1,\, \cdots,\, T$.
\subsection{Image Restoration by Conditional Diffusion}
An Image Restoration is to recover an original image $\x_0\in\R^n$ from a distorted version $\y\in\R^m$. In other words, we search for $\x_0$ such that
\begin{equation}
    \y = \A \x_0 + \n \, ,
\end{equation}
where $\y$ is a given noisy measurement degraded by a linear operator $\A:\R^n\rightarrow\R^m$ and a Gaussian noise $\n\sim\Nc(\boldsymbol{0},\sigma^2\I)$. A conventional method for tackling this challenge entails optimizing the following objective function:
\begin{equation} \label{eqn:opt}
    \argmin{\x} \norm{\y-\A\x}_2^2 + \lambda \mathcal{R}(\x)\, .
\end{equation}
The first term of (\ref{eqn:opt}) is the data-fidelity term, which quantifies the pixel-level dissimilarities between the measurement and the distorted version of the generated image. Whereas the second term is regularization term.

Sampling of $\x_0$ given $\y$ can be performed using the conditional generation of diffusion model \cite{scoresde}. The corresponding reverse SDE of (\ref{VPSDE}) is given as follows:
\begin{equation}
    d\x = \left[-\frac{\beta (t)}{2}\x -\beta(t) \nabla_{\x} \log p_t (\x|\y)\right] dt + \sqrt{\beta (t)}d\bar{\w} \, .
\end{equation}
Adopting the convention of representing $\x_t$ the value of $\x$ at $t$, it follows from the Bayes' rule that the conditional score can be further decomposed into two terms: 
\begin{equation}
    \gradxt \log p_t(\x_t|\y) = \gradxt \log p_t(\x_t) + \gradxt \log p_t(\y|\x_t)\, .
\end{equation}
For the unconditional score $\gradxt \log p_t(\x_t)$, a pre-trained score approximator $\s_\theta(\x_t,t)$ can be used. \cite{beatgan, isir} estimate the second term by training a model on paired data, whereas Diffusion Posterior Sampling (DPS) \cite{dps} approximates it by $\gradxt \log p (\y|\xzt)$ where $\xzt$ is the estimated posterior mean $\Eb [\x_0 | \x_t]$, and hence 
\begin{equation}
    \gradxt \log p_t(\y|\x_t)\ \simeq -\frac{1}{\sigma^2} \gradxt \norm{\y-\A\xzt}_2^2 \, .
\end{equation}
\section{Proposed Method}

\begin{figure*}[t!]
  \centering
    \centerline{{\includegraphics[width=0.98\linewidth]{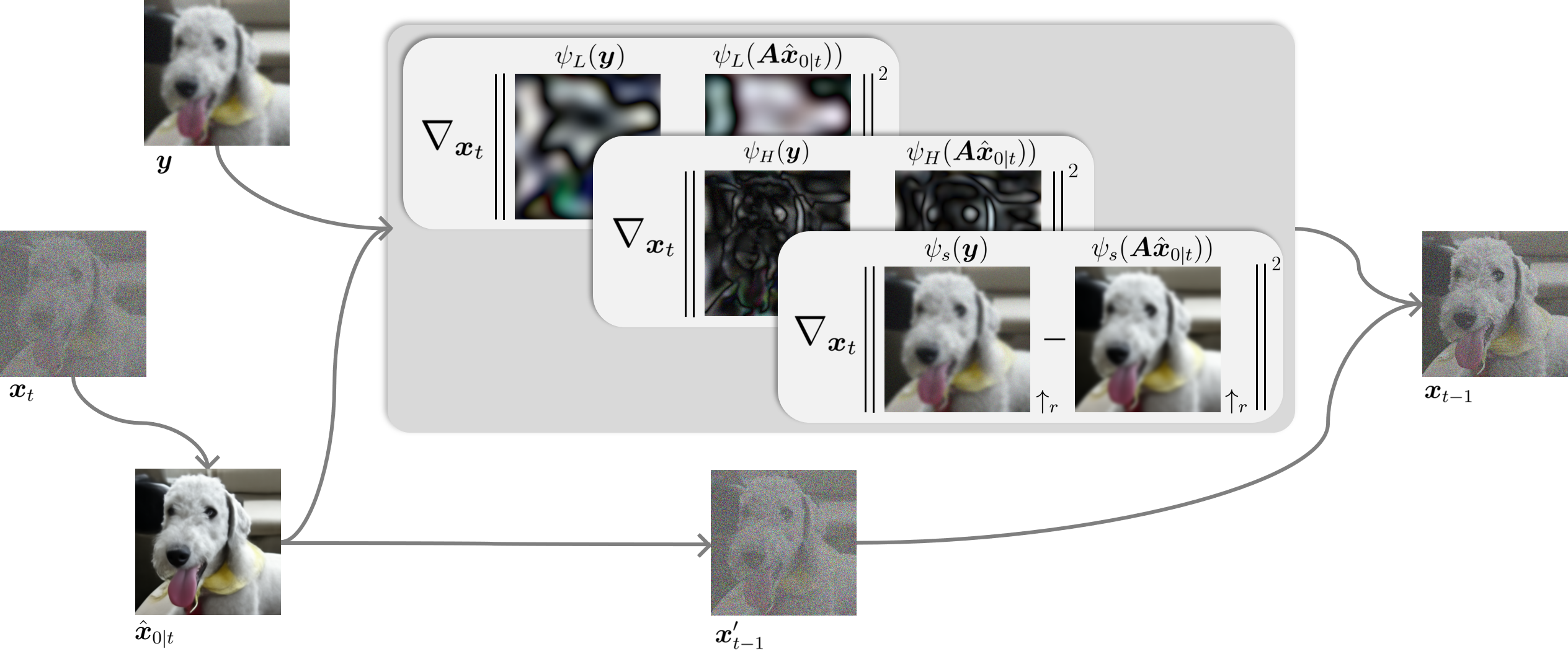}}}
  \caption{{The overview of {\name}.} 
  Starting with the intermediate state $\x_t$, we first generate the unconditional prediction $\xzt$ using the diffusion model. 
  Then we obtain the next state $\x_{t-1}$ by leveraging the loss guidance terms obtained through bicubic upsampling $\psi_{s}$ with scaling factor $r$, high-pass filter $\psi_H$ and the low-pass filter $\psi_L$.
  }
  \label{fig:mosikdo}
\end{figure*}

\label{sec:proposed_method}

Existing approaches to inverse problems \cite{dps,diffPIR,ddnm,ddrm,mcg} guide the generation process by minimizing the pixel-based data-fidelity term $ \lVert \y - \A\x \rVert _2 ^2$, which does not consider the perceptual features of images. In this paper, we propose a modified data-fidelity term that incorporates spatial and frequency features through upsampling and Fourier transformation, respectively.

\subsection{Modifying Data-fidelity}

With the aim of enhancing the data-fidelity term, we replace term $\norm{\y - \A \x_0}$ with term $\norm{\psi(\y) - \psi(\A \x_0 )}$ utilizing transformation $\psi$, which efficiently captures the perceptual characteristics of images.

To achieve satisfactory outcomes with this technique, it is recommended to carefully control the perturbation of the feasible solutions, aiming for minimal disruption.
Fortunately, the validity of the following equation is well-established under the condition that $\psi$ is injective.
\begin{equation}\label{eqn:argmin}
    \argmin{\x_0} \norm{\y - \A \x_0} = \argmin{\x_0} \norm{ \psi(\y) - \psi( \A \x_0)}
\end{equation}
Thus under the assumption of $\psi$'s injectivity, it is reasonable to expect that feasible solutions of (\ref{eqn:opt}) will be subjected to less disruption. Subsequently, we introduce an injective transformation $\psi$ that effectively extracts perceptual features.

\noindent
\textbf{Spatial feature~} We propose to leverage the effect of upsampling images obtained via interpolation to incorporate not only pixel-level information but also the spatial context of the image. Image interpolation employs convolutional operations to compute the values of newly generated pixels. These values are determined through intricate interactions with neighboring pixels, effectively capturing the spatial context of the image.

 In essence, the values of newly generated pixels can be interpreted as encapsulating crucial information derived from the surrounding spatial image patches.  By doing so, we aim to enrich the overall representation of the image, leading to improved performance in various image processing tasks.
 In this paper, we employ the standard interpolation method, bicubic interpolation. We denote $\psi_{s,r}$ the bicubic upsampling with the ratio $r$. {Note that the bicubic upsampling $\psi_s$ is injective.
\\
\noindent
\textbf{Frequency feature~} To enhance the alignment of the measurement with human perception, we incorporated the frequency domain representation of the image obtained through DFT.
Employing its strength in extracting frequency information, the discrete Fourier transform (DFT) empowers the decomposition of the data-fidelity term into its low-frequency and high-frequency counterparts, providing a more detailed representation of the data.

Denote $\mathcal{F}$ and $\mathcal{F}^{-1}$ as the 2D discrete fourier transform (DFT) and its inverse transform, respectively. 
For an image $f \in {\mathbb{R}}^{M \times N \times C}$, the discrete Fourier transform $\mathcal{F}$ decomposes $f$ by the orthonormal basis with complex coefficients as follows:
\begin{align}
    \left\{ \mathcal{F}(f) \right\} (u,v) = \sum_{i=0}^{M-1} \sum_{j=0}^{N-1} f(i,j)e^{-i 2\pi \left( \frac{ui}{M} + \frac{vj}{N} \right)}
\end{align}
for $(u,v) \in {\mathbb{R}^M\times \mathbb{R}^N}$ .  Our analysis builds upon channel-wise application of the DFT, henceforth represented without the channel dimension for conciseness.

We adopt the ideal highpass filtering and ideal lowpass filtering, denoted by $H$ and $L$, as follows:
\begin{align}
    \{H(F)\}_{u,v} = \begin{cases}
      0 & r(u,v) < r_0 \\
      F_{u,v} & \text{otherwise}
    \end{cases}\label{eq:12}   \\
    \{L(F)\}_{u,v} = \begin{cases}
      F_{u,v} & r(u,v) < r_0 \\
      0 & \text{otherwise}
    \end{cases}  \label{eq:13}
\end{align}
where $r(u,v) = \max \left\{ \left\lvert u - \frac{N}{2} \right\rvert, \left\lvert v - \frac{M}{2} \right\rvert \right\}$. 

Now we consider the transformation $\psi_{f} = ( \psi_{H} ^\top , \psi_{L}^\top )^\top $ where $\psi_H = \mathcal{F}^{-1} \circ H\circ \mathcal{F} $ and $\psi_L = \mathcal{F}^{-1} \circ L \circ \mathcal{F} $. The Parseval's theorem implies that $\psi_{f}$ preserves $2$-norm. Namely, denoting the difference $\y-\A \x_0$ as $\bm{d}$, $\lVert \psi_{f}(\bm{d}) \rVert _2 ^2 = \lVert \bm{d} \rVert _2 ^2$ holds. Hence, $\psi_{f}$ decomposes $\lVert \bm{d} \rVert _2 ^2$ to the high-frequency term $ \lVert   \psi_H (\bm{d}) \rVert _2 ^2$ and the low-frequency term $\lVert  \psi_L (\bm{d})  \rVert _2 ^2 $:
\begin{align}
    \lVert \bm{d} \rVert _2 ^2 = \lVert \psi_{f}(\bm{d}) \rVert _2 ^2 = \lVert  ( \psi_H (\bm{d}),\psi_L (\bm{d}) ) \rVert _2 ^2
    = \lVert   \psi_H (\bm{d}) \rVert _2 ^2 +\lVert  \psi_L (\bm{d})  \rVert _2 ^2
\end{align}

Note that the operator $\psi_f$ is norm-preserving operator, so $\psi_f$ is injective. Also, that minimizing $\lVert \psi_H(\bm{d}) \rVert _2 ^2$ and $\lVert \psi_L(\bm{d}) \rVert _2 ^2$ implies minimizing the difference of high frequency features and low frequency features, respectively. Therefore, through adaptive weighting of the decomposed fidelity terms, we can selectively enhance the high-frequency components that play a critical role in visual perception.

\subsection{Theoretical Analysis}

While the approximation of $p_t(\y | \x_t)$ in \cite{dps} is based on the assumption of a normal likelihood, the same principle can be extended to more general conditional probability distributions, represented by $\exp (- l_{\y}(\x_0))$ for some loss function $l_{\y}(\x_0)$ \cite{pnppriors, lgdm}. In the context of this paper, we focus on the specific case of $l_{\psi,\y}(\x_0) = \frac{1}{2\gamma^2} \left\lVert \psi (\y) - \psi( \A \x_0 ) \right\rVert_2 ^2 $

Same as \cite{dps}, we can factorize $p_t(\y | \x_t) $ as follows:
\begin{align}
     p_t(\y|\x_t) &= \int p_t(\y| \x_0, \x_t) p(\x_0 | \x_t) d\x_0 \\
     &= \int p(\y | \x_0 )p(\x_0 | \x_t) d\x_0
\end{align}
Our approach is to modify $p(\y | \x_0 )$ as $p_\psi(\y | \x_0 )$:
\begin{align}
\label{eq:14}
    p_\psi(\y | \x_0) = \frac{1}{Z_\psi} \exp \left[ -\frac{1}{2\gamma^2} \left\lVert \psi (\y) - \psi (\A \x_0) \right\rVert_2 ^2 \right] \, ,
\end{align}
where $Z_\psi$ is a normalizing constant. In that case, we factorize $p_{\psi,t} (\y | \x_t)$ by
\begin{align}
     p_{\psi,t} (\y|\x_t) =\int p_\psi (\y | \x_0 )p(\x_0 | \x_t) d\x_0
\end{align}

Assuming that $\psi$ is a linear operator, the distribution function $p_{\psi}(\y | \x_0)$ is Lipschitz continuous.

\begin{restatable}[]{lemma}{lemmm}
\label{lem}
  The modified conditional probability $p_{\psi}(\y | \x_0)$ defined as (\ref{eq:14})
  is Lipschitz continuous with respect to $\x_0$.

\end{restatable}

Denote the posterior mean of $p(\x_t | \x_0)$ as $\xzt=\mathbb{E}_{p(\x_0|\x_t)}[\x_0]$ where $p(\x_t | \x_0)$ is a DDPM forward process starting from time step 0.   
Then with similar argument in DPS \cite{dps}, the following holds:


\begin{restatable}[]{theorem}{main}
\label{thm:main}
  For the given measurement $\y$, a linear operator $\psi$, a modified conditional probability $p_\psi (\y | \x_0) \propto \exp ( -l_{\psi, \y} (\x_0) )$, 
we can approximate $p_{\psi,t} (\y | \x_t)$ as follows:
\begin{align}
    p_{\psi,t} (\y | \x_t ) \simeq p_\psi( \y | \xzt )
\end{align}
where the approximation error is bounded by
\begin{align}
    \lvert p_{\psi,t} (\y | \x_t ) - p_\psi( \y | \xzt ) \rvert \le \frac{1}{ {e}^{1/2}Z_{\psi} \gamma }   \cdot L_\psi \cdot \lVert \A \rVert \cdot m_1
\end{align}
where  $m_1 := \int \lVert \x_0 - \xzt \rVert p(\x_0 | \x_t) d\x_0$, $L_\psi$ is a Lipschitz constant of $\psi$ and $\lVert \A \rVert$ is the operator norm associated to the Euclidean norm. 

\end{restatable}

Furthermore, by the results of Theorem \ref{thm:main}., we can approximate gradient of log likelihood with the analytically tractable term:
\begin{align}
\nabla _{\x_t} \log p_{\psi,t} (\y | \x_t) \simeq \nabla_{\x_t} \log p_\psi (\y | \xzt ).
\label{align:gradient_}
\end{align}

\begin{remark}
Applying {Tweedie's} formula, one can prove that for the case of DDPM sampling, $\xzt$ has the explicit representation:
\begin{align}\label{eqn:xzt}
    \xzt =  \frac{1}{\sqrt{\bar{\alpha}(t)}} \left( \x_t + (1- \bar{\alpha}(t)) \gradxt \log p_t (\x_t) \right)
\end{align}

Note that $p_{\psi,t}(\y, \xzt)$ is intractable in general. Despite the inherent complexity of the term, using the Theorem \ref{thm:main} and the equation \ref{eqn:xzt}, we can approximate it into an explicit form.

\end{remark}

\subsection{\name}

Leveraging the synergistic power of the preceding concepts, we propose \textbf{\name}: Spatial-and-Frequency-aware Restoration method for Images, a novel methodology that tailors the data-fidelity term to spatial and frequency domains, enabling a more comprehensive representation of the underlying perceptual attributes of the images.

In order to consider both spatial and frequency features, we consider the data-fidelity with respect to $\psi = (\psi_{s}^\top,\psi_{f}^\top)^\top$. Since both $\psi_s$ and $\psi_f$ are injective, $\psi$ is also injective. Owing to its injective nature, $\psi$ is expected to cause minimal disruption to feasible solutions of (\ref{eqn:opt}).
In this case, it is represented as follow: 
\begin{align}
    \lVert \psi(\y) - \psi( \A \xzt) \rVert_2 ^2 = &\lVert \psi_s(\y) - \psi_s( \A \xzt) \rVert_2 ^2 \nonumber \\
    + &\lVert \psi_H(\y) - \psi_H( \A \xzt) \rVert_2 ^2 \nonumber \\
    + &\lVert \psi_L(\y) - \psi_L( \A \xzt) \rVert_2 ^2.
\end{align}

To enhance the algorithm's stability, in practice, we fix $\tau$ and set $\psi_s$ to identity for first $T-\tau$ iterations where $T$ is the total number of iterations. Additionally, To optimize its performance, we carefully adjusted the weights of the three data-fidelity terms: spatial-aware term $\rho_t^s$, high-frequency term $\rho_t^H$, and low-frequency term $\rho_t^L$.
A detailed algorithmic formulation of {\name} is presented in Algorithm \ref{alg:sf}. Choices of operators and hyperparameters are in Appendix. For the visual representation of {\name}, please refer to Figure \ref{fig:mosikdo}.

\begin{algorithm}[!t]
\caption{\name}
\begin{algorithmic}[1]

\Require 
\Statex{Total number of iterations : $T$, measurement : $\y$}
\Statex {Operators : $\psi_{s,r}$, $\psi_f = (\psi_{H}^\top,\psi_{L}^\top)^\top$}
\Statex{Hyperparameters : $\tau$}, ${\{\rho_t ^s, \rho_t ^H, \rho_t ^L \}_{t=1}^T}$, ${\{\tilde\sigma_t\}_{t=1}^T}$

\State{$\x_T \sim \Nc(\bm{0},\I)$}
\For{$t=T$ {\bfseries to} $1$}
    \State {$\psi_{s} \gets \psi_{s,r}$}
  \If{$t > \tau$}
    \State {$\psi_{s} \gets \text{identity}$}
  \EndIf
 \State{$\xzt \gets \frac{1}{\sqrt{\alphabar_t}} \left(\x_t + (1-\alphabar_t ) \s_\theta(\x_t, t) \right)$}
 \If{$t>1$}
    \State {$\z \sim \Nc(\bm{0},\I)$}
 \Else
    \State {$\z=\bm{0}$}
 \EndIf 
 \State{$\x'_{t-1} \gets \frac{\sqrt{\alpha_t}(1-\alphabar_{t-1})}{1-\alphabar_t}\x_t + \frac{\sqrt{\alphabar_{t-1}}\beta_t}{1-\alphabar_t}\xzt + \Tilde{\sigma}_i \z$}
 
  \State{$\Lc_{s} \gets \norm{\psi_{s}(\y) - \psi_{s}(\A (\xzt) ) }_2^2 $}
  \State{$\Lc_{H} \gets \norm{\psi_{H}(\y) - \psi_{H}(\A (\xzt) ) }_2^2 $}\State{$\Lc_{L} \gets \norm{\psi_{L}(\y) - \psi_{L}(\A (\xzt) ) }_2^2 $}

        \State {$\x_{t-1} \gets \x'_{t-1} - \rho_{t}^{s} \nabla_{\x_t} \Lc_{s} -  \rho_{t}^{H} \nabla_{\x_t} \Lc_{H} - \rho_{t}^{L} \nabla_{\x_t} \Lc_{L} $}

\EndFor
\State {\bfseries return} $\x_0$ 

\end{algorithmic}\label{alg:sf}
\end{algorithm}

\begin{table*}[ht]
\centering
\setlength{\tabcolsep}{0.2em} \resizebox{0.95\textwidth}{!}
{
\begin{tabular}{l@{\hskip 10pt}c@{\hskip 5pt}r@{\hskip 10pt}c@{\hskip 5pt}r@{\hskip 10pt}c@{\hskip 5pt}r@{\hskip 10pt}c@{\hskip 5pt}r@{\hskip 5pt}}
\toprule
\multirow{2}{*}{\textbf{Method}} & \multicolumn{2}{c}{\textbf{Inpaint (random)}}   &
\multicolumn{2}{c}{\textbf{Inpaint (box)}} & \multicolumn{2}{c}{\textbf{Deblur (Gauss)}}
& \multicolumn{2}{c}{\textbf{SR ($\times4$)}}

\\
\cmidrule(lr){2-3}
\cmidrule(lr){4-5}
\cmidrule(lr){6-7}
\cmidrule(lr){8-9}
 & \multicolumn{1}{c}{LPIPS $\downarrow$} & \multicolumn{1}{c}{FID $\downarrow$} & \multicolumn{1}{c}{LPIPS $\downarrow$} & \multicolumn{1}{c}{FID $\downarrow$} & \multicolumn{1}{c}{LPIPS $\downarrow$} & \multicolumn{1}{c}{FID $\downarrow$} & \multicolumn{1}{c}{LPIPS $\downarrow$} & \multicolumn{1}{c}{FID $\downarrow$} 
\\
\midrule
DPS \cite{dps} 
& 0.350 & 15.809 & 0.350 & 57.584 & 0.398 & 49.480 & {0.324} & {41.090}
\\
DiffPIR \cite{diffPIR} 
& {0.141} & {15.216}  & {0.255} & {47.210} & {0.336} & {39.502} & 0.351 & 44.176
\\
PnP-ADMM \cite{pnp-admm}
& 0.414 & 78.639 & 0.395 & 125.608 & 0.501 & 101.900 & 0.389 & 66.539
\\
ILVR \cite{ilvr}
& 0.352 & 48.419 & 0.315 & 61.083 & 0.477 & 80.369 & 0.441 & 74.364
\\
\cmidrule{1-9}
{\name} (ours)
& \textbf{0.124} & \textbf{10.477} & \underline{0.204} & \textbf{38.160} & \textbf{0.311} & \textbf{34.460} & \underline{0.296} & \textbf{31.258}
\\
{\name}-spatial (ours)
& \underline{0.127} & \underline{11.317} & 0.214 & 40.668 & \underline{0.312} & \underline{34.705} & 0.307 & 33.315
\\
{\name}-freq. (ours)
& 0.129 & 12.319 & \textbf{0.199} & \underline{38.358} & 0.317 & 35.160 & \textbf{0.295} & \underline{31.284}
\\
\bottomrule
\end{tabular}
}
\caption{
Quantitative evaluation of image restoration task with Gaussian noise ($\sigma=0.025$) on ImageNet 256$\times$256-1k validation dataset. We compare our method with other zero-shot IR methods. We compute the metrics LPIPS and FID for various tasks. \textbf{Bold}: Best, \underline{under}: second best. (The ranking was done before the rounding)
}
\label{tab:imagenet_quantitative}
\end{table*}

\begin{table*}[ht]
\centering
\setlength{\tabcolsep}{0.2em} \resizebox{0.95\textwidth}{!}
{
\begin{tabular}{l@{\hskip 10pt}c@{\hskip 5pt}r@{\hskip 10pt}c@{\hskip 5pt}r@{\hskip 10pt}c@{\hskip 5pt}r@{\hskip 10pt}c@{\hskip 5pt}r@{\hskip 5pt}}
\toprule
\multirow{2}{*}{\textbf{Method}}   & \multicolumn{2}{c}{\textbf{Inpaint (random)}}&
\multicolumn{2}{c}{\textbf{Inpaint (box)}} & \multicolumn{2}{c}{\textbf{Deblur (Gauss)}}
& \multicolumn{2}{c}{\textbf{SR ($\times4$)}}

\\
\cmidrule(lr){2-3}
\cmidrule(lr){4-5}
\cmidrule(lr){6-7}
\cmidrule(lr){8-9}
 & \multicolumn{1}{c}{LPIPS $\downarrow$} & \multicolumn{1}{c}{FID $\downarrow$} & \multicolumn{1}{c}{LPIPS $\downarrow$} & \multicolumn{1}{c}{FID $\downarrow$} & \multicolumn{1}{c}{LPIPS $\downarrow$} & \multicolumn{1}{c}{FID $\downarrow$} & \multicolumn{1}{c}{LPIPS $\downarrow$} & \multicolumn{1}{c}{FID $\downarrow$} 
\\
\midrule
DPS \cite{dps} 
& {0.099} & {12.766} & {0.147} & {20.885} & {0.211} & 23.152 & {0.200} & \underline{21.060}
\\
DiffPIR \cite{diffPIR} 
&  0.130 & 18.788 & 0.199 & 22.963 & 0.213 & {22.124} & 0.236 & 25.192
\\
PnP-ADMM \cite{pnp-admm} 
& 0.466 & 119.409 & 0.451 & 155.291 & 0.392 & 69.767 & 0.290 & 50.679
\\
ILVR \cite{ilvr} 
& 0.245 & 39.202 & 0.272 & 32.640 & 0.316 & 51.153 & 0.315 & 49.227
\\
\cmidrule{1-9}
{\name} (ours)
& \textbf{0.089} & \textbf{9.014} & \textbf{0.106} & \textbf{12.486} & \textbf{0.203} & \textbf{22.046} & \textbf{0.200} & \textbf{20.977}
\\
{\name}-spatial (ours)
& 0.098 & 10.954 & 0.130 & 14.338 & 0.206 & 22.567 & 0.204 & 21.938
\\
{\name}-freq. (ours)
& \underline{0.091} & \underline{9.800} & \underline{0.126} & \underline{13.000} & \underline{0.203} & \underline{22.083} & \underline{0.200} & {21.242} 
\\
\bottomrule
\end{tabular}
}
\caption{
Quantitative evaluation of image restoration task with Gaussian noise ($\sigma=0.025$) on FFHQ 256$\times$256-1k validation dataset. We compare our method with other zero-shot IR methods. We compute the metrics LPIPS \cite{lpips} and FID \cite{fid} for various tasks. 
\textbf{Bold}: Best, \underline{under}: second best. (The ranking was done before the rounding)
}
\label{tab:ffhq_quantitative}
\end{table*}

\section{Experiments}
\label{sec:experiments}

\subsection{Implementation Details}
To benchmark the proposed method against existing approaches, we conduct a comparative study using ImageNet $256\times 256$ \cite{imagenet} and FFHQ $256\times 256$ \cite{ffhq} datasets. For each dataset, we evaluate $1$k validation images. We leveraged the pre-trained diffusion models for ImageNet and FFHQ datasets taken from \cite{dhariwal2021diffusion} and \cite{dps}, respectively, without any adjustments. All images are normalized to the range $[0,1]$. 

We conducted experiments using four different degradation tasks. For inpainting, we utilize two mask type; box-type mask and random-type mask. For random-type mask, we mask out $92$\% of the total pixels, and for box-type mask, we mask out $128 \times 128$ box region randomly, following \cite{dps}. For gaussian deblur, we use $61\times61$ gaussian blur kernel with standard deviation of $3.0$. Lastly, For super-resolution, bicubic downsampling is performed. For all tasks, we add gaussian noise to the measurement with standard deviation of $0.025$. Additional experimental details are provided in Appendix.

\subsection{Quantitative Experiments}

To objectively assess the perceptual similarity between two images, we employ two widely recognized metrics: Fréchet Inception Distance (FID) and Learned Perceptual Image Patch Similarity (LPIPS). For comprehensive evaluation, additional objective metrics such as peak signal-to-noise ratio (PSNR) and structural similarity index (SSIM) are presented in the Appendix.

We test our approaches with the following methods:
Denoising diffusion models for plug-and-play image restoration(DiffPIR) \cite{diffPIR}, Diffusion posterior sampling for general noisy inverse problems(DPS) \cite{dps}, Plug-and-play alternating direction method of multipliers(PnP-ADMM) \cite{pnp-admm} using DnCNN \cite{zhang2017beyond} instead of proximal mappings, Iterative latent variable refinement(ILVR) \cite{ilvr}. Although ILVR only deals with super-resolution task, we adopted projections onto convex sets (POCS) method \cite{dps} for inpainting \cite{scoresde} and Gaussian deblurring task. In the case of DiffPIR, we used the case with NFEs of $100$ for the experiments. We additionally evaluate {\name}-spatial and {\name}-frequency, which are {\name} methods with  $\psi_f=0$ and $\psi_s=0$ respectively. To ensure a fair comparison, same pre-trained score function are used for all diffusion model-based methods in the experiments.

We present the quantitative results for ImageNet dataset in Table \ref{tab:imagenet_quantitative} and results for FFHQ dataset in Table \ref{tab:ffhq_quantitative}. Our method establishes a new state-of-the-art benchmark for image restoration performance, consistently outperforming other zero-shot based methods across both datasets and all image restoration tasks.
In particular, our method achieved significantly better results than previous methods on the ImageNet dataset, where the data prior is more complex and conditional guiding is more important.
Furthermore, the utilization of either spatial or frequency alone is sufficient for our method to outperform existing methods, demonstrating our approach's superior capabilities and effectiveness.

\subsection{Qualitative Experiments}
\begin{figure*}[t]
  \centering
    \centerline{{\includegraphics[width=0.95\linewidth]{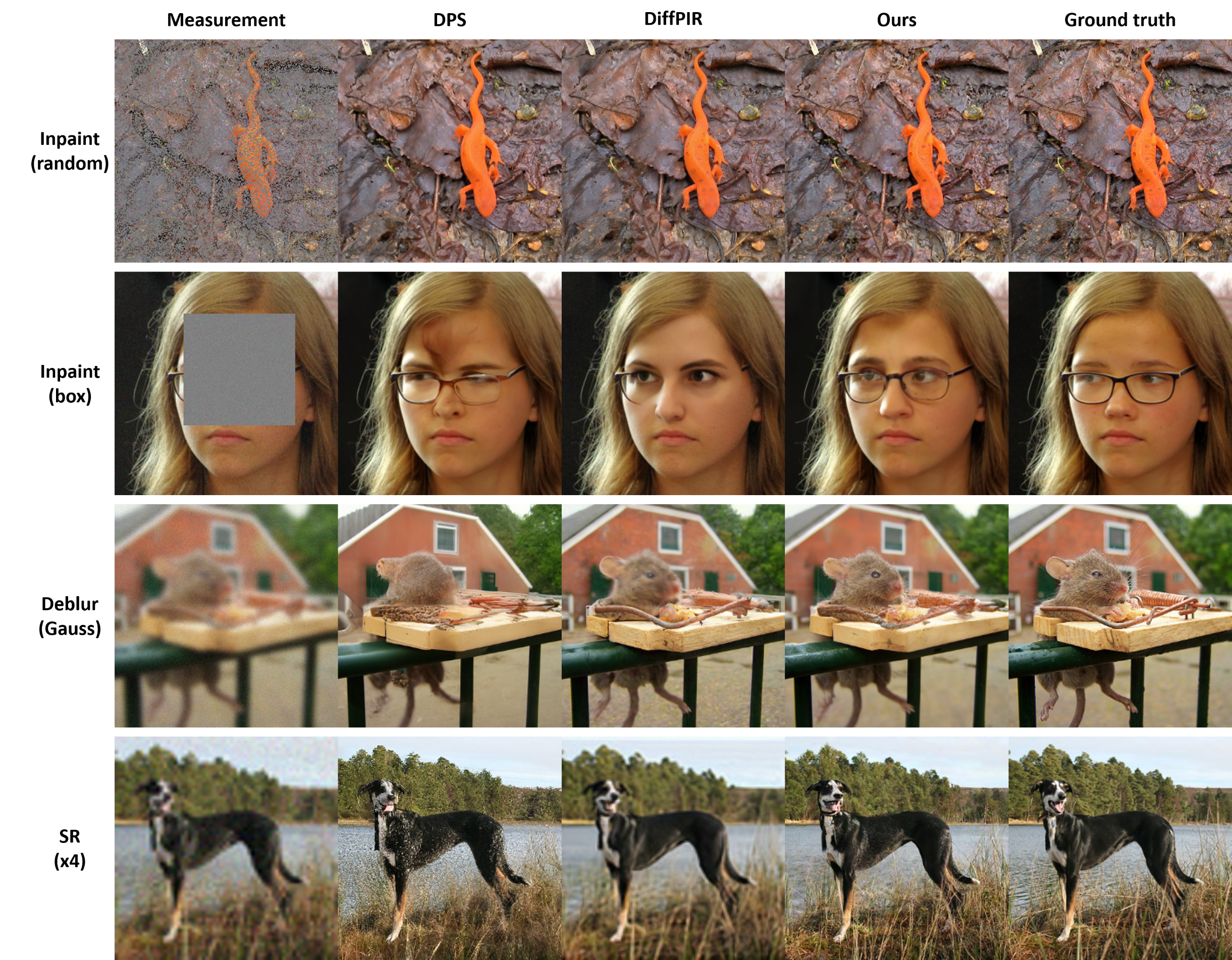}}}
  \caption{Qualitative results of image restoration. We establish the efficacy of {\name} in restoring images across a variety of tasks.
  }
  \label{fig:results}
\end{figure*}
Figure \ref{fig:intro_img} illustrates the superior performance of our proposed approach. Also, figure \ref{fig:results} provides a comprehensive comparison of {\name} against DPS and DiffPIR on a variety of IR tasks, encompassing inpainting with random and box masks, Gaussian deblurring, and super-resolution ($\times 4$). Our method consistently surpasses the benchmarks, generating high-quality reconstructions that exhibit remarkable perceptual context alignment. \\
For random-type inpainting, previous methods also generate sufficiently good images, but our method better restores semantic features. 
A comparison of the first row of Fig \ref{fig:results} with results obtained from other methods clearly highlights the superior performance of our method in generating highly detailed and realistic representations of the surrounding environment and the lizard's skin.
Also, despite the inherent complexities of box inpainting, our approach successfully generates realistic and seamlessly integrated glasses occluded by the mask, resulting in natural-looking images with imperceptible boundaries between the masked and unmasked regions.
In addition, our approach successfully deblurs the measurements, leading to the faithful reconstruction of intricate features, such as the eyes and fur texture.
Moreover, in the context of super-resolution applications, our method excels at reconstructing high-quality images that maintain the integrity of both foreground objects and background elements.

Overall, {\name} can effectively capture and preserve fine-grained texture information, resulting in more realistic and visually appealing images.
\section{Conclusion}
\label{sec:conclusion}

This paper presents {\textbf{\name}}, a novel diffusion model-based image restoration approach that incorporates spatial and frequency information into the data fidelity term, effectively enhancing the restoration performance. By leveraging both spatial and frequency via bicubic upsampling and Fourier transformation, {\textbf{\name}} achieves state-of-the-art results on a variety of image restoration benchmarks, outperforming existing methods.

Despite the remarkable performance of our proposed methodology, the application of the transformation inevitably induces perturbations to the feasible solutions due to the influence of the prior term. A comprehensive analysis of the solution perturbations can strengthen the theoretical foundation of our methodology.

\bibliographystyle{unsrtnat}
\bibliography{references} 

\begin{thebibliography}{42}
\providecommand{\natexlab}[1]{#1}
\providecommand{\url}[1]{\texttt{#1}}
\expandafter\ifx\csname urlstyle\endcsname\relax
  \providecommand{\doi}[1]{doi: #1}\else
  \providecommand{\doi}{doi: \begingroup \urlstyle{rm}\Url}\fi

\bibitem[Rudin et~al.(1992)Rudin, Osher, and Fatemi]{rudin}
Leonid~I Rudin, Stanley Osher, and Emad Fatemi.
\newblock Nonlinear total variation based noise removal algorithms.
\newblock \emph{Physica D: nonlinear phenomena}, 60\penalty0 (1-4):\penalty0 259--268, 1992.

\bibitem[Benning and Burger(2018)]{datafidreg}
Martin Benning and Martin Burger.
\newblock Modern regularization methods for inverse problems.
\newblock \emph{Acta numerica}, 27:\penalty0 1--111, 2018.

\bibitem[Ho et~al.(2020)Ho, Jain, and Abbeel]{ddpm}
Jonathan Ho, Ajay Jain, and Pieter Abbeel.
\newblock Denoising diffusion probabilistic models.
\newblock \emph{Advances in neural information processing systems}, 33:\penalty0 6840--6851, 2020.

\bibitem[Song et~al.(2021)Song, Sohl-Dickstein, Kingma, Kumar, Ermon, and Poole]{scoresde}
Yang Song, Jascha Sohl-Dickstein, Diederik~P Kingma, Abhishek Kumar, Stefano Ermon, and Ben Poole.
\newblock Score-based generative modeling through stochastic differential equations.
\newblock In \emph{International Conference on Learning Representations}, 2021.
\newblock URL \url{https://openreview.net/forum?id=PxTIG12RRHS}.

\bibitem[Dhariwal and Nichol(2021{\natexlab{a}})]{beatgan}
Prafulla Dhariwal and Alexander Nichol.
\newblock Diffusion models beat gans on image synthesis.
\newblock \emph{Advances in neural information processing systems}, 34:\penalty0 8780--8794, 2021{\natexlab{a}}.

\bibitem[Choi et~al.(2021)Choi, Kim, Jeong, Gwon, and Yoon]{ilvr}
Jooyoung Choi, Sungwon Kim, Yonghyun Jeong, Youngjune Gwon, and Sungroh Yoon.
\newblock Ilvr: Conditioning method for denoising diffusion probabilistic models.
\newblock \emph{arXiv preprint arXiv:2108.02938}, 2021.

\bibitem[Kawar et~al.(2021)Kawar, Vaksman, and Elad]{snips}
Bahjat Kawar, Gregory Vaksman, and Michael Elad.
\newblock {SNIPS}: Solving noisy inverse problems stochastically.
\newblock \emph{Advances in Neural Information Processing Systems}, 34:\penalty0 21757--21769, 2021.

\bibitem[Kawar et~al.(2022{\natexlab{a}})Kawar, Elad, Ermon, and Song]{ddrm}
Bahjat Kawar, Michael Elad, Stefano Ermon, and Jiaming Song.
\newblock Denoising diffusion restoration models.
\newblock \emph{Advances in Neural Information Processing Systems}, 35:\penalty0 23593--23606, 2022{\natexlab{a}}.

\bibitem[Wang et~al.(2022)Wang, Yu, and Zhang]{ddnm}
Yinhuai Wang, Jiwen Yu, and Jian Zhang.
\newblock Zero-shot image restoration using denoising diffusion null-space model.
\newblock \emph{arXiv preprint arXiv:2212.00490}, 2022.

\bibitem[Chung et~al.(2023)Chung, Kim, Mccann, Klasky, and Ye]{dps}
Hyungjin Chung, Jeongsol Kim, Michael~Thompson Mccann, Marc~Louis Klasky, and Jong~Chul Ye.
\newblock Diffusion posterior sampling for general noisy inverse problems.
\newblock In \emph{The Eleventh International Conference on Learning Representations}, 2023.
\newblock URL \url{https://openreview.net/forum?id=OnD9zGAGT0k}.

\bibitem[Zhu et~al.(2023)Zhu, Zhang, Liang, Cao, Wen, Timofte, and Gool]{diffPIR}
Yuanzhi Zhu, Kai Zhang, Jingyun Liang, Jiezhang Cao, Bihan Wen, Radu Timofte, and Luc~Van Gool.
\newblock Denoising diffusion models for plug-and-play image restoration.
\newblock In \emph{IEEE Conference on Computer Vision and Pattern Recognition Workshops (NTIRE)}, 2023.

\bibitem[Chung et~al.(2022{\natexlab{a}})Chung, Sim, and Ye]{ccdf}
Hyungjin Chung, Byeongsu Sim, and Jong~Chul Ye.
\newblock Come-closer-diffuse-faster: Accelerating conditional diffusion models for inverse problems through stochastic contraction.
\newblock In \emph{Proceedings of the IEEE/CVF Conference on Computer Vision and Pattern Recognition}, pages 12413--12422, 2022{\natexlab{a}}.

\bibitem[Chung et~al.(2022{\natexlab{b}})Chung, Sim, Ryu, and Ye]{mcg}
Hyungjin Chung, Byeongsu Sim, Dohoon Ryu, and Jong~Chul Ye.
\newblock Improving diffusion models for inverse problems using manifold constraints.
\newblock \emph{Advances in Neural Information Processing Systems}, 35:\penalty0 25683--25696, 2022{\natexlab{b}}.

\bibitem[Song et~al.(2022)Song, Vahdat, Mardani, and Kautz]{pgdm}
Jiaming Song, Arash Vahdat, Morteza Mardani, and Jan Kautz.
\newblock Pseudoinverse-guided diffusion models for inverse problems.
\newblock In \emph{International Conference on Learning Representations}, 2022.

\bibitem[Lugmayr et~al.(2022)Lugmayr, Danelljan, Romero, Yu, Timofte, and Van~Gool]{repaint}
Andreas Lugmayr, Martin Danelljan, Andres Romero, Fisher Yu, Radu Timofte, and Luc Van~Gool.
\newblock Repaint: Inpainting using denoising diffusion probabilistic models.
\newblock In \emph{Proceedings of the IEEE/CVF Conference on Computer Vision and Pattern Recognition}, pages 11461--11471, 2022.

\bibitem[Fei et~al.(2023)Fei, Lyu, Pan, Zhang, Yang, Luo, Zhang, and Dai]{GDP}
Ben Fei, Zhaoyang Lyu, Liang Pan, Junzhe Zhang, Weidong Yang, Tianyue Luo, Bo~Zhang, and Bo~Dai.
\newblock Generative diffusion prior for unified image restoration and enhancement.
\newblock In \emph{Proceedings of the IEEE/CVF Conference on Computer Vision and Pattern Recognition}, pages 9935--9946, 2023.

\bibitem[Kawar et~al.(2022{\natexlab{b}})Kawar, Song, Ermon, and Elad]{jpegddrm}
Bahjat Kawar, Jiaming Song, Stefano Ermon, and Michael Elad.
\newblock Jpeg artifact correction using denoising diffusion restoration models.
\newblock \emph{arXiv preprint arXiv:2209.11888}, 2022{\natexlab{b}}.

\bibitem[Chen and Ford(1994)]{chen1994perceptual}
Hong Chen and Gary~E Ford.
\newblock Perceptual wiener filtering for image restoration.
\newblock In \emph{Proceedings of 1994 28th Asilomar Conference on Signals, Systems and Computers}, volume~1, pages 218--222. IEEE, 1994.

\bibitem[Jam et~al.(2021)Jam, Kendrick, Drouard, Walker, Hsu, and Yap]{jam2021r}
Jireh Jam, Connah Kendrick, Vincent Drouard, Kevin Walker, Gee-Sern Hsu, and Moi~Hoon Yap.
\newblock R-mnet: A perceptual adversarial network for image inpainting.
\newblock In \emph{Proceedings of the IEEE/CVF Winter Conference on Applications of Computer Vision}, pages 2714--2723, 2021.

\bibitem[Lukac(2017)]{lukac2017perceptual}
Rastislav Lukac.
\newblock \emph{Perceptual digital imaging: methods and applications}.
\newblock CRC Press, 2017.

\bibitem[Wang et~al.(2018)Wang, Xu, Wang, and Tao]{wang2018perceptual}
Chaoyue Wang, Chang Xu, Chaohui Wang, and Dacheng Tao.
\newblock Perceptual adversarial networks for image-to-image transformation.
\newblock \emph{IEEE Transactions on Image Processing}, 27\penalty0 (8):\penalty0 4066--4079, 2018.

\bibitem[Yang et~al.(2020)Yang, Wang, Fang, Wang, and Liu]{yang2020fidelity}
Wenhan Yang, Shiqi Wang, Yuming Fang, Yue Wang, and Jiaying Liu.
\newblock From fidelity to perceptual quality: A semi-supervised approach for low-light image enhancement.
\newblock In \emph{Proceedings of the IEEE/CVF conference on computer vision and pattern recognition}, pages 3063--3072, 2020.

\bibitem[Ma et~al.(2022)Ma, Liu, and Wu]{ma2022rectified}
Haichuan Ma, Dong Liu, and Feng Wu.
\newblock Rectified wasserstein generative adversarial networks for perceptual image restoration.
\newblock \emph{IEEE Transactions on Pattern Analysis and Machine Intelligence}, 45\penalty0 (3):\penalty0 3648--3663, 2022.

\bibitem[Li et~al.(2016)Li, Yan, Fang, Wang, Tang, and Qian]{li2016perceptual}
Leida Li, Ya~Yan, Yuming Fang, Shiqi Wang, Lu~Tang, and Jiansheng Qian.
\newblock Perceptual quality evaluation for image defocus deblurring.
\newblock \emph{Signal Processing: Image Communication}, 48:\penalty0 81--91, 2016.

\bibitem[Yang et~al.(2023)Yang, Zhou, Feng, and Wang]{dpdmms}
Xingyi Yang, Daquan Zhou, Jiashi Feng, and Xinchao Wang.
\newblock Diffusion probabilistic model made slim.
\newblock In \emph{Proceedings of the IEEE/CVF Conference on Computer Vision and Pattern Recognition}, pages 22552--22562, 2023.

\bibitem[Yuan et~al.(2023)Yuan, Li, Wang, Yang, Lin, Liu, and Wang]{sfunet}
Xin Yuan, Linjie Li, Jianfeng Wang, Zhengyuan Yang, Kevin Lin, Zicheng Liu, and Lijuan Wang.
\newblock Spatial-frequency u-net for denoising diffusion probabilistic models.
\newblock \emph{arXiv preprint arXiv:2307.14648}, 2023.

\bibitem[Sheng et~al.(2022)Sheng, Liu, Cao, Shen, and Zhang]{fddgid}
Zehua Sheng, Xiongwei Liu, Si-Yuan Cao, Hui-Liang Shen, and Huaqi Zhang.
\newblock Frequency-domain deep guided image denoising.
\newblock \emph{IEEE Transactions on Multimedia}, 2022.

\bibitem[Fuoli et~al.(2021)Fuoli, Van~Gool, and Timofte]{fuoli2021fourier}
Dario Fuoli, Luc Van~Gool, and Radu Timofte.
\newblock Fourier space losses for efficient perceptual image super-resolution.
\newblock In \emph{Proceedings of the IEEE/CVF International Conference on Computer Vision}, pages 2360--2369, 2021.

\bibitem[Ayyoubzadeh and Wu(2021)]{ayyoubzadeh2021high}
Seyed~Mehdi Ayyoubzadeh and Xiaolin Wu.
\newblock High frequency detail accentuation in cnn image restoration.
\newblock \emph{IEEE Transactions on Image Processing}, 30:\penalty0 8836--8846, 2021.

\bibitem[Shao et~al.(2023)Shao, Qiao, Meng, and Zuo]{shao2023uncertainty}
Mingwen Shao, Yuanjian Qiao, Deyu Meng, and Wangmeng Zuo.
\newblock Uncertainty-guided hierarchical frequency domain transformer for image restoration.
\newblock \emph{Knowledge-Based Systems}, 263:\penalty0 110306, 2023.

\bibitem[Anderson(1982)]{anderson}
Brian~D.O. Anderson.
\newblock {Reverse-time diffusion equation models}.
\newblock \emph{Stochastic Processes and their Applications}, 12\penalty0 (3):\penalty0 313--326, May 1982.
\newblock URL \url{https://ideas.repec.org/a/eee/spapps/v12y1982i3p313-326.html}.

\bibitem[Song and Ermon(2019)]{ncsn}
Yang Song and Stefano Ermon.
\newblock Generative modeling by estimating gradients of the data distribution.
\newblock In \emph{Advances in Neural Information Processing Systems}, pages 11895--11907, 2019.

\bibitem[Saharia et~al.(2022)Saharia, Ho, Chan, Salimans, Fleet, and Norouzi]{isir}
Chitwan Saharia, Jonathan Ho, William Chan, Tim Salimans, David~J Fleet, and Mohammad Norouzi.
\newblock Image super-resolution via iterative refinement.
\newblock \emph{IEEE Transactions on Pattern Analysis and Machine Intelligence}, 45\penalty0 (4):\penalty0 4713--4726, 2022.

\bibitem[Graikos et~al.(2022)Graikos, Malkin, Jojic, and Samaras]{pnppriors}
Alexandros Graikos, Nikolay Malkin, Nebojsa Jojic, and Dimitris Samaras.
\newblock Diffusion models as plug-and-play priors.
\newblock \emph{Advances in Neural Information Processing Systems}, 35:\penalty0 14715--14728, 2022.

\bibitem[Song et~al.(2023)Song, Zhang, Yin, Mardani, Liu, Kautz, Chen, and Vahdat]{lgdm}
Jiaming Song, Qinsheng Zhang, Hongxu Yin, Morteza Mardani, Ming-Yu Liu, Jan Kautz, Yongxin Chen, and Arash Vahdat.
\newblock Loss-guided diffusion models for plug-and-play controllable generation.
\newblock In Andreas Krause, Emma Brunskill, Kyunghyun Cho, Barbara Engelhardt, Sivan Sabato, and Jonathan Scarlett, editors, \emph{Proceedings of the 40th International Conference on Machine Learning}, volume 202 of \emph{Proceedings of Machine Learning Research}, pages 32483--32498. PMLR, 23--29 Jul 2023.
\newblock URL \url{https://proceedings.mlr.press/v202/song23k.html}.

\bibitem[Chan et~al.(2016)Chan, Wang, and Elgendy]{pnp-admm}
Stanley~H Chan, Xiran Wang, and Omar~A Elgendy.
\newblock Plug-and-play admm for image restoration: Fixed-point convergence and applications.
\newblock \emph{IEEE Transactions on Computational Imaging}, 3\penalty0 (1):\penalty0 84--98, 2016.

\bibitem[Zhang et~al.(2018)Zhang, Isola, Efros, Shechtman, and Wang]{lpips}
Richard Zhang, Phillip Isola, Alexei~A Efros, Eli Shechtman, and Oliver Wang.
\newblock The unreasonable effectiveness of deep features as a perceptual metric.
\newblock In \emph{Proceedings of the IEEE conference on computer vision and pattern recognition}, pages 586--595, 2018.

\bibitem[Heusel et~al.(2017)Heusel, Ramsauer, Unterthiner, Nessler, and Hochreiter]{fid}
Martin Heusel, Hubert Ramsauer, Thomas Unterthiner, Bernhard Nessler, and Sepp Hochreiter.
\newblock Gans trained by a two time-scale update rule converge to a local nash equilibrium.
\newblock \emph{Advances in neural information processing systems}, 30, 2017.

\bibitem[Deng et~al.(2009)Deng, Dong, Socher, Li, Li, and Fei-Fei]{imagenet}
Jia Deng, Wei Dong, Richard Socher, Li-Jia Li, Kai Li, and Li~Fei-Fei.
\newblock Imagenet: A large-scale hierarchical image database.
\newblock In \emph{2009 IEEE conference on computer vision and pattern recognition}, pages 248--255. Ieee, 2009.

\bibitem[Karras et~al.(2019)Karras, Laine, and Aila]{ffhq}
Tero Karras, Samuli Laine, and Timo Aila.
\newblock A style-based generator architecture for generative adversarial networks.
\newblock In \emph{Proceedings of the IEEE/CVF conference on computer vision and pattern recognition}, pages 4401--4410, 2019.

\bibitem[Dhariwal and Nichol(2021{\natexlab{b}})]{dhariwal2021diffusion}
Prafulla Dhariwal and Alexander Nichol.
\newblock Diffusion models beat gans on image synthesis.
\newblock \emph{Advances in neural information processing systems}, 34:\penalty0 8780--8794, 2021{\natexlab{b}}.

\bibitem[Zhang et~al.(2017)Zhang, Zuo, Chen, Meng, and Zhang]{zhang2017beyond}
Kai Zhang, Wangmeng Zuo, Yunjin Chen, Deyu Meng, and Lei Zhang.
\newblock Beyond a gaussian denoiser: Residual learning of deep cnn for image denoising.
\newblock \emph{IEEE transactions on image processing}, 26\penalty0 (7):\penalty0 3142--3155, 2017.

\end{thebibliography}

\clearpage
\setcounter{page}{1}
\maketitlesupplementary

\def\thesection{\Alph{section}}
\setcounter{section}{0}
\section{Proofs}

\begin{lemma1}{}
\label{lem:lemma_prime}
  The modified conditional probability $p_{\psi}(\y | \x_0)$ defined as (\ref{eq:14})
  is Lipschitz continuous with respect to $\x_0$.
  
\end{lemma1}
\begin{proof}
    Let $h(\z) := \frac{1}{Z_\psi} \exp \left[ -\frac{1}{2\gamma^2} \lVert \z - \psi(\y)  \rVert^2 _2 \right]$, then the modified conditional distribution $p_{\psi}(\y | \x_0)$ is a composition of $h$, $\psi$ and $\A$. Since $\psi$ and $\A$ are both linear, it suffices to show that $h$ is Lipschitz continuous. Since
    \begin{align}
        \frac{\partial}{\partial \z_i} h(\z) &= \frac{1}{Z_\psi} \cdot \left( -\frac{\z_i - \psi(\y)_i}{\gamma^2} \right) \cdot \exp \left[ -\frac{1}{2\gamma^2} \lVert \z - \psi(\y)  \rVert^2 _2 \right],   \\ 
        \frac{\partial^2}{\partial \z_i ^2} h(\z) &= \frac{1}{Z_\psi \gamma^2} \cdot \left[ \frac{(\z_i - \psi(\y)_i)^2}{\gamma^2} -1 \right] \cdot \exp \left[ - \frac{1}{2\gamma^2}\lVert \z - \psi(\y)  \rVert^2 _2 \right],
    \end{align}
    for each $\z_i$, $\left | \frac{\partial h}{\partial \z_i }\right |$ attains maximum $\frac{1}{e^{1/2}Z_\psi \gamma}$ at $\z_i = \psi(\y)_i \pm \gamma$. Thus we have 
    \begin{align}
        \lvert h(\z_1) - h(\z_2) \rvert &\le \sup_{\z} \lVert \nabla h(\z) \rVert_{\infty} \cdot \lVert \z_1 - \z_2 \rVert \\
        &\le \sup_{\z} \  \max_i\left | \frac{\partial h}{\partial \z_i}\right | \cdot \lVert \z_1 - \z_2 \rVert \\
        &\le \frac{1}{e^{1/2}Z_\psi \gamma}  \lVert \z_1 - \z_2 \rVert\, ,
    \end{align}
    which shows that $h$ is Lipschitz continuous and consequently $p_{\psi}(\y | \x_0)$ is also Lipschitz continuous. Note that the Lipschitz constant of $p_{\psi}(\y | \x_0)$ is the product of all Lipschitz constants of $h$, $\psi$ and $\A$.
\end{proof}
Employing the aforementioned above lemma, we can derive the following theorem:

\main*
\begin{proof} 
    It follows from the above lemma that $p_{\psi}(\y | \x_0)$ is Lipschitz continuous with Lipschitz constant $\frac{1}{ {e}^{1/2}Z_{\psi} \gamma } \cdot L_{\psi} \cdot \norm{\A}$. Therefore, the error bound is 
    \begin{align}
        \lvert p_{\psi,t} (\y | \x_t ) - p_\psi( \y | \xzt ) \rvert  &= \snorm{\int p_{\psi}(\y|\x_0)p(\x_0|\x_t)d\x_0 - p_{\psi}(\y|\xzt) } \\
        &= \snorm{ \int \left( p_{\psi} (\y | \x_0 ) - p_\psi( \y | \xzt )\right)p(\x_0|\x_t) d\x_0} \\
        &\le \int \snorm{\left( p_{\psi} (\y | \x_0 ) - p_\psi( \y | \xzt )\right)}p(\y|\xzt)d\x_0 \\
        &\le \frac{1}{ {e}^{1/2}Z_{\psi} \gamma }   \cdot L_{\psi} \cdot \norm{\A} \cdot \int \norm{\x_0 - \xzt}p(\x_0|\x_t)d\x_0 \\
        &=\frac{1}{ {e}^{1/2}Z_{\psi} \gamma } \cdot L_{\psi} \cdot \norm{\A} \cdot m_1\, ,
    \end{align}
    where the second last line comes from the Lipschitz continuity.
\end{proof}
    
\clearpage

\section{Experimental Details}
\subsection{Hyperparameter setting}

We provide a comprehensive overview of the hyperparameter configurations utilized for our algorithm in each problem setting. For spatial feature operator $\psi_{s,r}$, bicubic upsampling with a factor of $r=4$ has been used throughout all experiments. Table \ref{tab:hyper-param-imagenet} and Table \ref{tab:hyper-param-ffhq} shows the values for the hyperparameters for each task on ImageNet and FFHQ respectively.
\begin{table}[ht]
\centering
\setlength{\tabcolsep}{0.2em} 
\setlength{\extrarowheight}{1pt}
{
\begin{tabular}
{c@{\hskip 10pt}c@{\hskip 10pt}c@{\hskip 10pt}c@{\hskip 10pt}c@{\hskip 10pt}}
\toprule
 & \multicolumn{1}{c}{\textbf{Inpainting (random)}}&
\multicolumn{1}{c}{\textbf{Inpainting (box)}} & \multicolumn{1}{c}{\textbf{Deblur (Gauss)}}
& \multicolumn{1}{c}{\textbf{SR ($\times4$)}}

\\
\midrule
$r_0$ & 5 & 5 & 4 & 5
\\
$\tau$ & 0.7 & 0.5 & 0.5 & 0.7
\\
$\rho_{t>\tau}^H$ & 0.0 / $\sqrt{\Lc_H}$  & 0.125 / $\sqrt{\Lc_H}$ & 0.0125 / $\sqrt{\Lc_H}$ & 0.25 / $\sqrt{\Lc_H}$
\\
$\rho_{t>\tau}^L$ & 0.0 / $\sqrt{\Lc_L}$ & 0.125 / $\sqrt{\Lc_L}$ & 0.025 / $\sqrt{\Lc_L}$ & 0.25 / $\sqrt{\Lc_L}$
\\
$\rho_{t>\tau}^s$& 0.25 / $\sqrt{\Lc_s}$ & 0.125 / $\sqrt{\Lc_s}$ & 0.075 / $\sqrt{\Lc_s}$ & 0.025 / $\sqrt{\Lc_s}$
\\
$\rho_{t\le\tau}^H$& 0.125 / $\sqrt{\Lc_H}$ & 0.625 / $\sqrt{\Lc_H}$ & 0.3 / $\sqrt{\Lc_H}$ & 1.25 / $\sqrt{\Lc_H}$
\\
$\rho_{t\le\tau}^L$& 0.025 / $\sqrt{\Lc_L}$ & 0.125 / $\sqrt{\Lc_L}$ & 0.15 / $\sqrt{\Lc_L}$ & 0.25 / $\sqrt{\Lc_L}$
\\
$\rho_{t\le\tau}^s$& 0.35 / $\sqrt{\Lc_s}$ & 0.125 / $\sqrt{\Lc_s}$ & 0.225 / $\sqrt{\Lc_s}$ & 0.0 / $\sqrt{\Lc_s}$
\\

\bottomrule
\end{tabular}
}
\caption{
Hyperparameters of image restoration tasks on ImageNet 256$\times$256 dataset.
}
\label{tab:hyper-param-imagenet}
\end{table}
\begin{table}[ht]
\centering
\setlength{\tabcolsep}{0.2em} 
\setlength{\extrarowheight}{1pt}
{
\begin{tabular}
{c@{\hskip 10pt}c@{\hskip 10pt}c@{\hskip 10pt}c@{\hskip 10pt}c@{\hskip 10pt}}
\toprule
 & \multicolumn{1}{c}{\textbf{Inpainting (random)}}&
\multicolumn{1}{c}{\textbf{Inpainting (box)}} & \multicolumn{1}{c}{\textbf{Deblur (Gauss)}}
& \multicolumn{1}{c}{\textbf{SR ($\times4$)}}

\\

\midrule
$r_0$ & 5 & 5 & 5 & 2
\\
$\tau$ & 0.7 & 0.5 & 0.7 & 0.7
\\
$\rho_{t>\tau}^H$ & 0.2 / $\sqrt{\Lc_H}$ & 0.125 / $\sqrt{\Lc_H}$ & 0.25 / $\sqrt{\Lc_H}$ & 0.15 / $\sqrt{\Lc_H}$
\\
$\rho_{t>\tau}^L$ & 0.2 / $\sqrt{\Lc_L}$ & 0.125 / $\sqrt{\Lc_L}$ & 0.25 / $\sqrt{\Lc_L}$ & 0.15 / $\sqrt{\Lc_L}$
\\
$\rho_{t>\tau}^s$& 0.075 / $\sqrt{\Lc_s}$ & 0.05 / $\sqrt{\Lc_s}$ & 0.05 / $\sqrt{\Lc_s}$ & 0.1 / $\sqrt{\Lc_s}$
\\
$\rho_{t\le\tau}^H$& 0.8 / $\sqrt{\Lc_H}$ & 0.75 / $\sqrt{\Lc_H}$  & 1.25 / $\sqrt{\Lc_H}$ & 1.0 / $\sqrt{\Lc_H}$
\\
$\rho_{t\le\tau}^L$& 0.2 / $\sqrt{\Lc_L}$ & 0.375 / $\sqrt{\Lc_L}$ & 0.25 / $\sqrt{\Lc_L}$ & 0.25 / $\sqrt{\Lc_L}$
\\
$\rho_{t\le\tau}^s$& 0.15 / $\sqrt{\Lc_s}$ & 0.1 / $\sqrt{\Lc_s}$ & 0.025 / $\sqrt{\Lc_s}$ & 0.0 / $\sqrt{\Lc_s}$
\\

%


\bottomrule
\end{tabular}
}
\caption{
Hyperparameters of image restoration tasks on FFHQ 256$\times$256 dataset.
}
\label{tab:hyper-param-ffhq}
\end{table}

In addition, for the Super Resolution task employing the FFHQ dataset, we utilize an upsampling operator in place of the identity operator for $\psi_s$.

\subsection{Comparison methods}
\noindent
\textbf{DPS} The default configuration in DPS \cite{dps} was employed for all experiments, except for a few tasks. Given the enhanced performance of the new setting for both Gaussian deblurring(GB) and box-mask inpainting(IB) tasks on the ImageNet dataset, we opted to adopt this configuration for subsequent experiments. In this regard, we established the hyperparameter as follows:

\begin{align}
    \zeta_{i} = \begin{cases}
      0.15/\norm{\y - \A(\xzt(\x_i)) } & \text{Deblur (Gauss), ImageNet dataset} \\
     0.25/\norm{\y - \A(\xzt(\x_i)) } & \text{Inpainting (box), ImageNet dataset}
    \end{cases}  
\end{align}
\\

\noindent
\textbf{DiffPIR} To maintain consistency in the noise level across our experiments, we employed the official code of DiffPIR (\cite{diffPIR}) and adjusted the $\mathrm{noise\ level}$ parameter to $6.375$.
\\

\noindent
\textbf{PnP-ADMM} We take the \url{scico} library's implementation for our proposes. We set the ADMM penalty parameter $\rho = 0.2$ and $\mathrm{maxiter}=12$. Also, we leverage the pretrained DnCNN denoiser \cite{zhang2017beyond} for proximal mapping.
\\

\noindent
\textbf{ILVR} For SR task, we followed ILVR and for other tasks we adopted projections onto convex sets (POCS) method, as in \cite{dps}, for inpainting \cite{scoresde} and Gaussian deblurring task.

\section{Ablation studies}

\subsection{Measurement Noise}

We will demonstrate the robustness of our proposed method to varying noise levels. We will showcase the effectiveness of our approach under diverse noise conditions, highlighting its resilience against noise interference. 
DiffPIR's reliance on a closed-form solution for conditional guiding grants it an advantage in noiseless scenarios. However, our method surpasses DiffPIR under varying noise levels, demonstrating its superior robustness.

\begin{table}[ht]
\centering
\setlength{\tabcolsep}{0.2em} 
{
\begin{tabular}{c@{\hskip 10pt}cccc}
\toprule


\multicolumn{1}{l}{Noise $\sigma$} &\multicolumn{1}{c}{0.0} &\multicolumn{1}{c}{0.025} & \multicolumn{1}{c}{0.05}  & \multicolumn{1}{c}{0.1}
\\
\cmidrule(lr){1-5}

DiffPIR
& 0.151 & 0.205 & 0.227 & 0.253
\\
SaFaRI & 0.154 & 0.193 & 0.212  & 0.242
\\

\bottomrule
\end{tabular}
}
\caption{
LPIPS comparison on measurement noise, FFHQ 256$\times$256 -100 validation dataset, Gaussian deblurring task. 
}
\vspace{-0.5em}
\label{tab:ablation_noise}
\end{table}

\subsection{SR Scaling Factor}
We evaluate our methods on the SR task using the LPIPS metric with various scaling factors. Table \ref{tab:abl:SRfactor} demonstrates that the task becomes increasingly challenging as the scaling factor increases, which aligns with expectations. 
\begin{table}[ht]
\centering
\setlength{\tabcolsep}{0.2em}  
{
\begin{tabular}{c@{\hskip 10pt}cccc}
\toprule
\multicolumn{1}{l}{Scaling Factor} & \multicolumn{1}{c}{$\times 2$} &\multicolumn{1}{c}{$\times 4$} & \multicolumn{1}{c}{$\times 8$} & \multicolumn{1}{c}{$\times 16$}
\\
\cmidrule(lr){1-5}
LPIPS&  0.116 & 0.191 & 0.270 & 0.247\\
\bottomrule
\end{tabular}
}
\caption{
LPIPS evaluation of our method across various scaling factors of the SR task on FFHQ 256$\times$256 -100 validation dataset. 
}
\vspace{-0.5em}
\label{tab:abl:SRfactor}
\end{table}

\subsection{Spatial Hyperparameters}
\noindent \textbf{Upsampling Factor}
To investigate the impact of varying upsampling factors $r$ of the bicubic operator $\psi_{s,r}$ on the performance of our method, we evaluated the LPIPS metric across different scaling factors. The results, presented in Table \ref{tab:abl:Spatfactor}, demonstrate a trend between the upsampling factor and LPIPS score.

\begin{table}[ht]
\centering
\setlength{\tabcolsep}{0.2em}  
{
\begin{tabular}{c@{\hskip 10pt}cccc}
\toprule
\multicolumn{1}{l}{Upsampling Factor} & \multicolumn{1}{c}{$\times 2$} &\multicolumn{1}{c}{$\times 4$} & \multicolumn{1}{c}{$\times 8$} & \multicolumn{1}{c}{$\times 16$}
\\
\cmidrule(lr){1-5}
LPIPS&  0.193 & 0.193 & 0.194 & 0.195\\
\bottomrule
\end{tabular}
}
\caption{
LPIPS evaluation of our method across various upsampling factors of spatial operator $\psi_{s,r}$ for Gaussian Deblurring task on FFHQ 256$\times$256 -100 validation dataset. 
}
\vspace{-0.5em}
\label{tab:abl:Spatfactor}
\end{table}

\subsection{Frequency Hyperparameters } 
\label{subsection: freq hyperparam}

\noindent
\textbf{Radius $r_0$} \
In order to analyze the frequency context of the predicted measurement, the frequency domain is segmented into high-frequency and low-frequency components. The Fast Fourier Transform (FFT) algorithm is employed to transform the measurement data into the frequency domain. By applying an appropriate frequency shift, the low-frequency components are centered in the transformed image, facilitating the separation of the frequency domain into two distinct regions as illustrated in (\ref{eq:12}) and (\ref{eq:13}).

Parameter $r_0$ plays a pivotal role in governing the emphasis placed on specific frequency components. Consequently, the selection of an appropriate radius parameter $r_0$ is crucial for highlighting the desired fine-grained features. Table \ref{tab:ablation_radius} demonstrates the impact of varying the $r_0$.

\noindent
\textbf{Emphasizing parameters $\rho_t^H$,$\rho_t^L$} \
Beyond the influence of $r_0$, $\rho_t^H$ and $\rho_t^L$ also plays a significant role in determining the relative emphasis placed on high-frequency contextual information compared to low-frequency contextual information. Table \ref{tab:ablation_high}, \ref{tab:ablation_low} illustrate the effect of altering the relative proportions of high-frequency and low-frequency components. 
Insufficient values of $\rho_t^H$ do not guarantee the generation of detailed images (Fig \ref{fig:high_param}). Additionally, excessive values of $\rho_t^L$ results in distorted and corrupted generated images (Fig \ref{fig:low_param}).
\begin{table}[ht]
\centering
\setlength{\tabcolsep}{0.2em} 
{
\begin{tabular}{c@{\hskip 10pt}ccccc}
\toprule
\multicolumn{1}{c}{$r_0$} & \multicolumn{1}{c}{1} &\multicolumn{1}{c}{2} & \multicolumn{1}{c}{3} & \multicolumn{1}{c}{4} & \multicolumn{1}{c}{5}
\\
\cmidrule(lr){1-6}

SaFaRI&  0.200 & 0.199 & 0.196 & 0.195 & 0.193\\
\bottomrule
\end{tabular}
}
\caption{
LPIPS evaluation on
radius $r_0$ of the Gaussian Deblurring task on FFHQ 256$\times$256 -100 validation dataset. 
}
\vspace{-0.5em}
\label{tab:ablation_radius}
\end{table}
\begin{table}[ht]
\centering
\setlength{\tabcolsep}{0.2em}
{
\begin{tabular}{c@{\hskip 10pt}ccccc}
\toprule
\multicolumn{1}{c}{$\rho_t^H \cdot 4\sqrt{\Lc_H}$  } & \multicolumn{1}{c}{1} &\multicolumn{1}{c}{2} & \multicolumn{1}{c}{3} & \multicolumn{1}{c}{4} & \multicolumn{1}{c}{5}
\\
\cmidrule(lr){1-6}

SaFaRI &  0.220 & 0.207 & 0.201 & 0.196 & 0.193\\

\bottomrule
\end{tabular}
}
\caption{
LPIPS evaluation on $\rho_t^H$, FFHQ 256$\times$256 -100 validation dataset, Gaussian deblurring task with fixed $\rho_t^L$. The case $\rho_t^H \cdot 4\sqrt{\Lc_H} = 5$ represents the optimal case.
}
\vspace{-0.5em}
\label{tab:ablation_high}
\end{table}

\begin{table}[ht]
\centering
\setlength{\tabcolsep}{0.2em}
{
\begin{tabular}{c@{\hskip 10pt}ccccc}
\toprule
\multicolumn{1}{c}{$\rho_t^L \cdot 4\sqrt{\Lc_L}$ } & \multicolumn{1}{c}{1} &\multicolumn{1}{c}{2} & \multicolumn{1}{c}{3} & \multicolumn{1}{c}{4} & \multicolumn{1}{c}{5}
\\
\cmidrule(lr){1-6}

SaFaRI & 0.193  & 0.196 & 0.206 & 0.222 & 0.236
\\

\bottomrule
\end{tabular}
}
\caption{
LPIPS evaluation on $\rho_t^L$, FFHQ 256$\times$256 -100 validation dataset, Gaussian deblurring task with fixed $\rho_t^H$. The case $\rho_t^L \cdot 4\sqrt{\Lc_L} = 1$  represents the optimized case.
}
\vspace{-0.5em}
\label{tab:ablation_low}
\end{table}


\begin{figure*}[ht]
  \centering
    \centerline{{\includegraphics[width=0.5\linewidth]{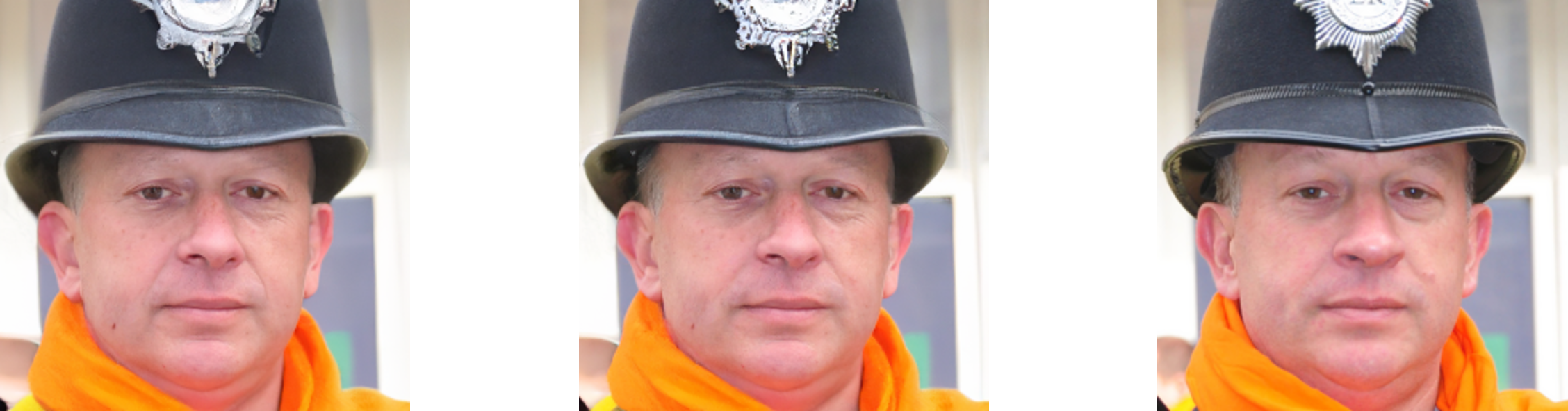}}}
  \caption{The results of SaFaRI, Gaussian blurring under different $\rho_t^H$ configurations. (left) The case $\rho_t^H = 0.25 / \sqrt{\Lc_H}$ (middle) The case $\rho_t^H = 1.25 / \sqrt{\Lc_H}$ (right) Ground Truth. 
  }
  \label{fig:high_param}
\end{figure*}

\begin{figure*}[ht]
  \centering
    \centerline{{\includegraphics[width=0.5\linewidth]{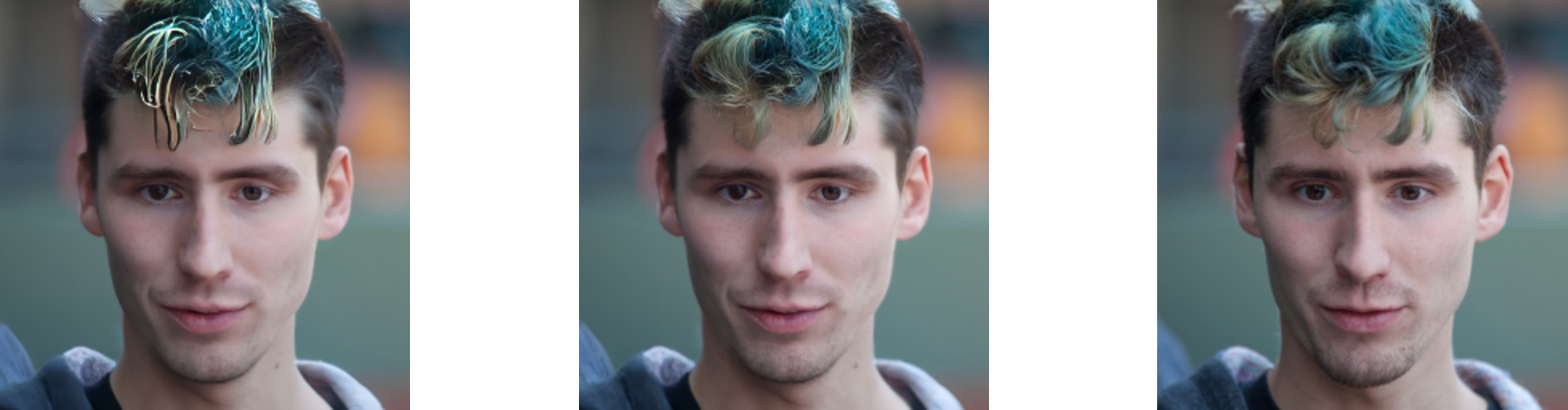}}}
  \caption{The results of SaFaRI, Gaussian blurring under different $\rho_t^H$ configurations. (left) The case $\rho_t^L = 1.25 / \sqrt{\Lc_L}$ (middle) The case $\rho_t^L = 0.25 / \sqrt{\Lc_L}$ (right) Ground Truth.
  }
  \label{fig:low_param}
\end{figure*}

\section{Further Experiments}

We conduct additional quantitative analyses employing the standard metrics, PSNR and SSIM in Table \ref{tab:imagenet_psnr_ssim} and Table \ref{tab:ffhq_psnr_ssim}. Also we present the further experimental results: from Figure \ref{fig:imnt_ir} to Figure\ref{fig:ffhq_sr}. Our method exhibits robust performance across a variety of tasks and datasets, as evidenced by comprehensive empirical evaluations.
\begin{table*}[ht]
\centering
\setlength{\tabcolsep}{0.2em} 
{
\begin{tabular}{l@{\hskip 10pt}c@{\hskip 5pt}r@{\hskip 10pt}c@{\hskip 5pt}r@{\hskip 10pt}c@{\hskip 5pt}r@{\hskip 10pt}c@{\hskip 5pt}r@{\hskip 5pt}}
\toprule
\multirow{2}{*}{\textbf{Method}} & \multicolumn{2}{c}{\textbf{Inpaint (random)}}   &
\multicolumn{2}{c}{\textbf{Inpaint (box)}} & \multicolumn{2}{c}{\textbf{Deblur (Gauss)}}
& \multicolumn{2}{c}{\textbf{SR ($\times4$)}}

\\
\cmidrule(lr){2-3}
\cmidrule(lr){4-5}
\cmidrule(lr){6-7}
\cmidrule(lr){8-9}
 & \multicolumn{1}{c}{PSNR $\uparrow$} & \multicolumn{1}{c}{SSIM $\uparrow$} & \multicolumn{1}{c}{PSNR $\uparrow$} & \multicolumn{1}{c}{SSIM $\uparrow$} & \multicolumn{1}{c}{PSNR $\uparrow$} & \multicolumn{1}{c}{SSIM $\uparrow$} & \multicolumn{1}{c}{PSNR $\uparrow$} & \multicolumn{1}{c}{SSIM $\uparrow$} 
\\
\midrule
DPS \cite{dps} 
& 28.699 & 0.875 & 18.472 & 0.667 & 21.134 & 0.532 & 21.898 & {0.601}
\\
DiffPIR \cite{diffPIR} 
& {28.504} & {0.862}  & {18.638} & {0.706} & \textbf{23.453} & {0.618} & \textbf{23.815} & 0.660
\\
PnP-ADMM \cite{pnp-admm}
& 18.467 & 0.584 & 13.310 & 0.579 & 21.788 & 0.595 & 23.444 & \textbf{0.702}
\\
ILVR \cite{ilvr}
& 23.402 & 0.605 & 16.868 & 0.615 & 21.290 & 0.579 & 23.355 & 0.627
\\
\cmidrule{1-9}
{\name} (ours)
& {28.732} & {0.880} & \textbf{18.887} & \underline{0.778} & {23.064} & \underline{0.642} & {23.703} & {0.680}
\\
{\name}-spatial (ours)
& \underline{28.753} & \underline{0.880} & {18.560} & {0.771} & \underline{23.098} & \textbf{0.643} & {23.646} & {0.672}
\\
{\name}-freq. (ours)
& \textbf{28.833} & \textbf{0.881} & \underline{18.855} & \textbf{0.779} & {22.862} & {0.629} & \underline{23.732} & \underline{0.680}
\\
\bottomrule
\end{tabular}
}
\caption{
Quantitative evaluation of image restoration task with Gaussian noise ($\sigma=0.025$) on ImageNet 256$\times$256-1k validation dataset. We compare our method with other zero-shot IR methods. We compute the metrics PSNR and SSIM for various tasks. \textbf{Bold}: Best, \underline{under}: second best. (The ranking was done before the rounding)
}
\label{tab:imagenet_psnr_ssim}
\end{table*}

\begin{table*}[ht]
\centering
\setlength{\tabcolsep}{0.2em} 
{
\begin{tabular}{l@{\hskip 10pt}c@{\hskip 5pt}r@{\hskip 10pt}c@{\hskip 5pt}r@{\hskip 10pt}c@{\hskip 5pt}r@{\hskip 10pt}c@{\hskip 5pt}r@{\hskip 5pt}}
\toprule
\multirow{2}{*}{\textbf{Method}}   & \multicolumn{2}{c}{\textbf{Inpaint (random)}}&
\multicolumn{2}{c}{\textbf{Inpaint (box)}} & \multicolumn{2}{c}{\textbf{Deblur (Gauss)}}
& \multicolumn{2}{c}{\textbf{SR ($\times4$)}}

\\
\cmidrule(lr){2-3}
\cmidrule(lr){4-5}
\cmidrule(lr){6-7}
\cmidrule(lr){8-9}
 & \multicolumn{1}{c}{PSNR $\uparrow$} & \multicolumn{1}{c}{SSIM $\uparrow$} & \multicolumn{1}{c}{PSNR $\uparrow$} & \multicolumn{1}{c}{SSIM $\uparrow$} & \multicolumn{1}{c}{PSNR $\uparrow$} & \multicolumn{1}{c}{SSIM $\uparrow$} & \multicolumn{1}{c}{PSNR $\uparrow$} & \multicolumn{1}{c}{SSIM $\uparrow$} 
\\
\midrule
DPS \cite{dps} 
& {32.370} & \underline{0.933} & 21.495 & 0.834 & 26.365 & 0.776 & \underline{27.539} & 0.812 
\\
DiffPIR \cite{diffPIR} 
&  31.345 & 0.912 & 21.928 & 0.783 & \textbf{27.420} & \textbf{0.796} & 27.120 & 0.778
\\
PnP-ADMM \cite{pnp-admm} 
& 18.634 & 0.605 & 12.587 & 0.560 & 24.746 & 0.759 & 26.412 & \textbf{0.834}
\\
ILVR \cite{ilvr} 
& 25.504 & 0.768 & 19.986 & 0.677 & 24.672 & 0.754 & 27.287 & 0.769
\\
\cmidrule{1-9}
{\name} (ours)
& \textbf{32.534} & 0.933 & \textbf{23.492} & \textbf{0.852} & \underline{26.725} & \underline{0.786} & 27.538 & 0.812
\\
{\name}-spatial (ours)
& 32.317 & 0.922 & 22.976 & 0.847 & 26.558 & 0.781 & 27.412 & 0.807
\\
{\name}-freq. (ours)
& \underline{32.496} & \textbf{0.934} & \underline{23.247} & \underline{0.848} & 26.709 & 0.786 & \textbf{27.556} & \underline{0.813}
\\
\bottomrule
\end{tabular}
}
\caption{
Quantitative evaluation of image restoration task with Gaussian noise ($\sigma=0.025$) on FFHQ 256$\times$256-1k validation dataset. We compare our method with other zero-shot IR methods. We compute the metrics PSNR and SSIM for various tasks. 
\textbf{Bold}: Best, \underline{under}: second best. (The ranking was done before the rounding)
}
\label{tab:ffhq_psnr_ssim}
\end{table*}
\begin{figure*}[t]
  \centering
    \centerline{{\includegraphics[width=0.9\linewidth]{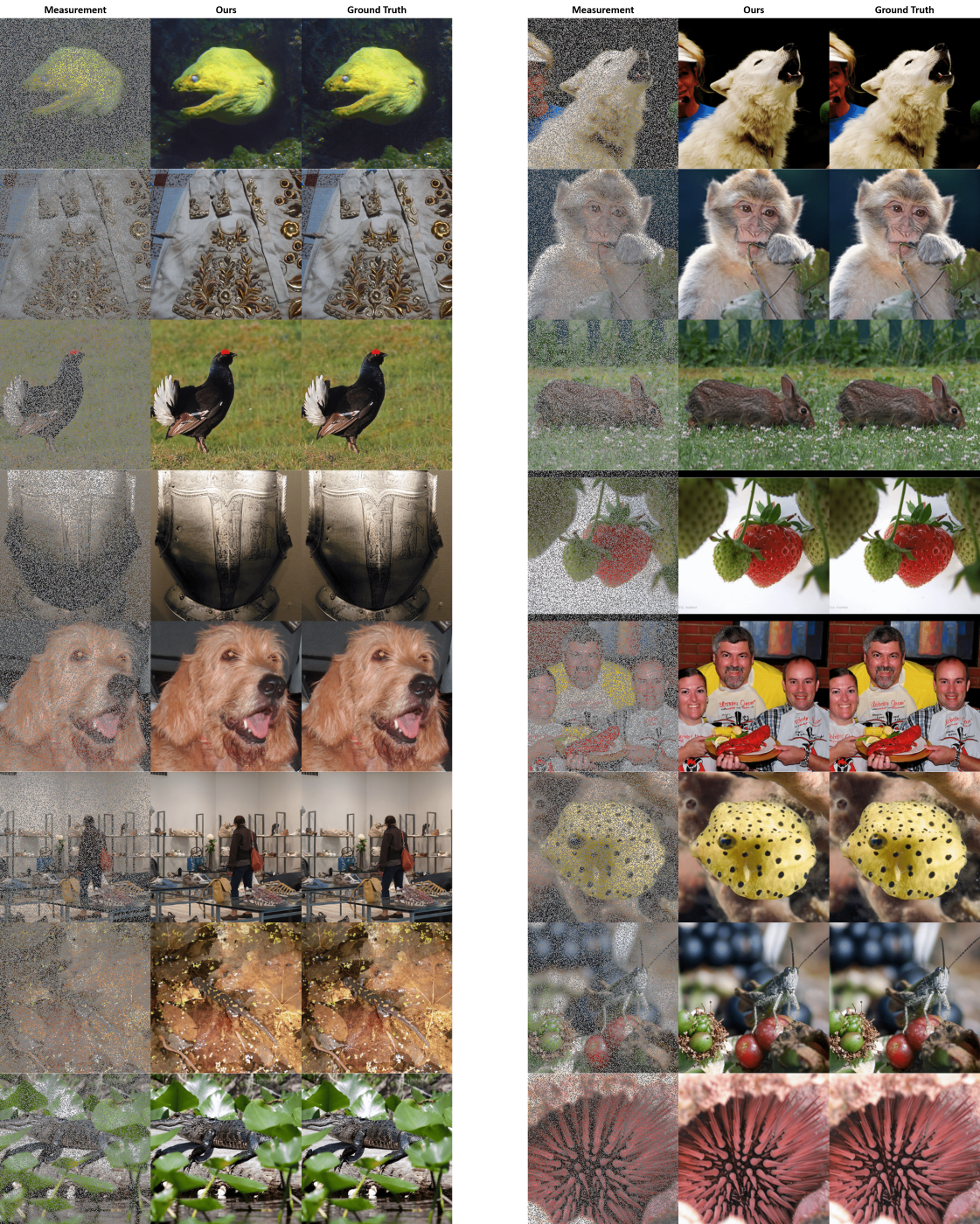}}}
  \caption{ The results of SaFaRI, random-mask inpainting on the ImageNet 256$\times$256 dataset. 
  }
  \label{fig:imnt_ir}
\end{figure*}

\begin{figure*}[t]
  \centering
    \centerline{{\includegraphics[width=0.9\linewidth]{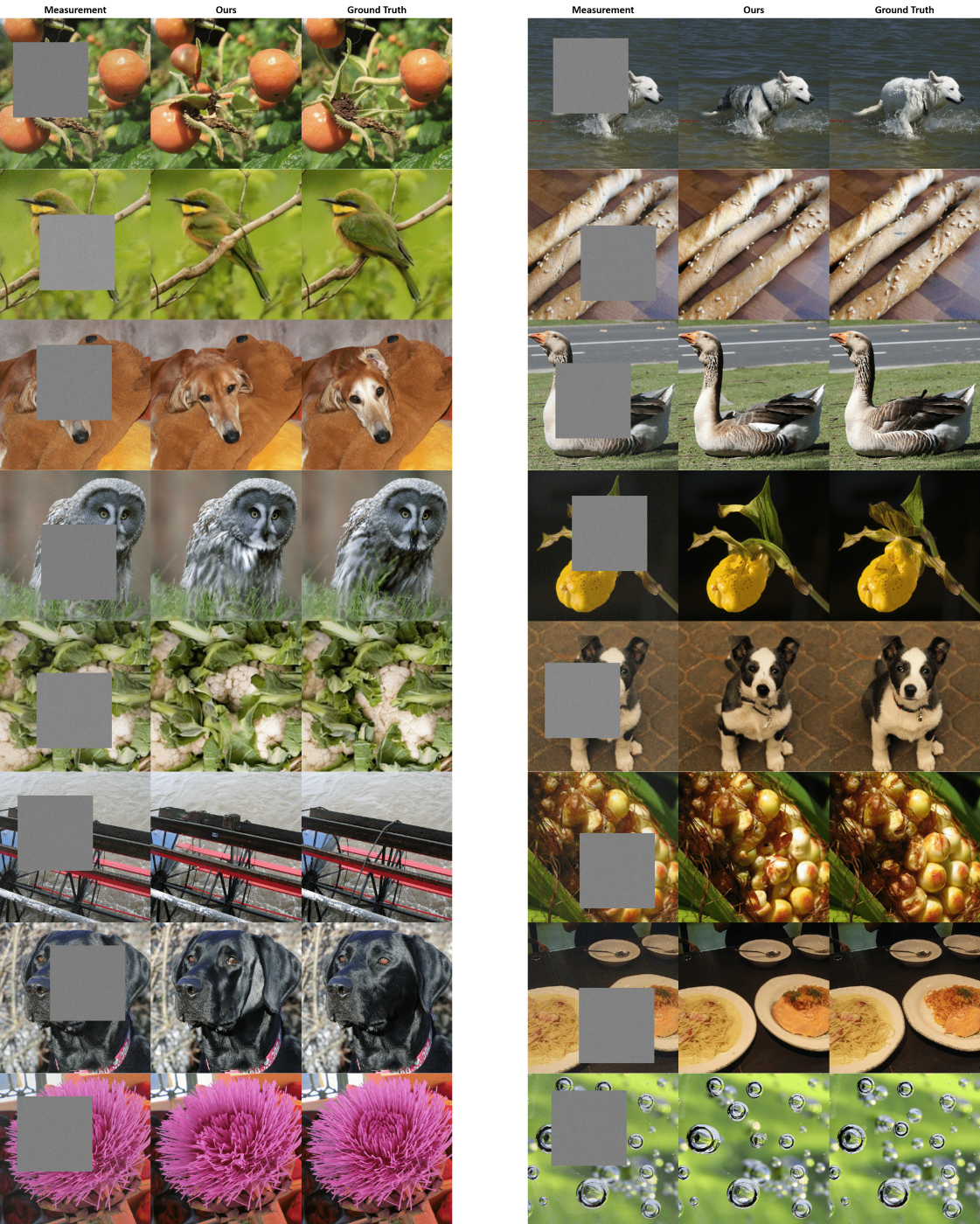}}}
  \caption{The results of SaFaRI, box-mask inpainting on the ImageNet 256$\times$256 dataset.
  }
  \label{fig:imnt_ib}
\end{figure*}

\begin{figure*}[t]
  \centering
    \centerline{{\includegraphics[width=0.9\linewidth]{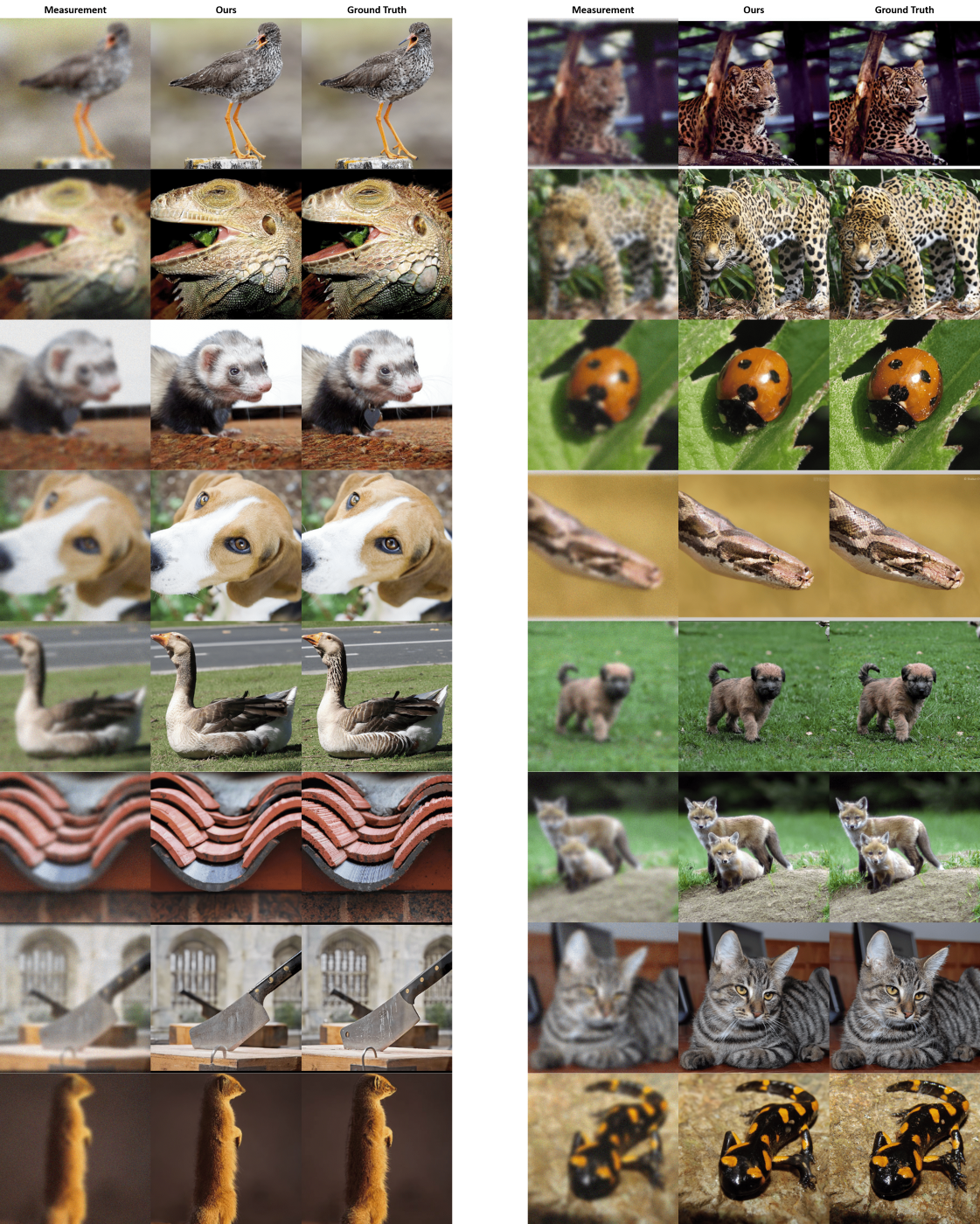}}}
  \caption{The results of SaFaRI, Gaussian deblurring on the ImageNet 256$\times$256 dataset.
  }
  \label{fig:imnt_gd}
\end{figure*}

\begin{figure*}[t]
  \centering
    \centerline{{\includegraphics[width=0.9\linewidth]{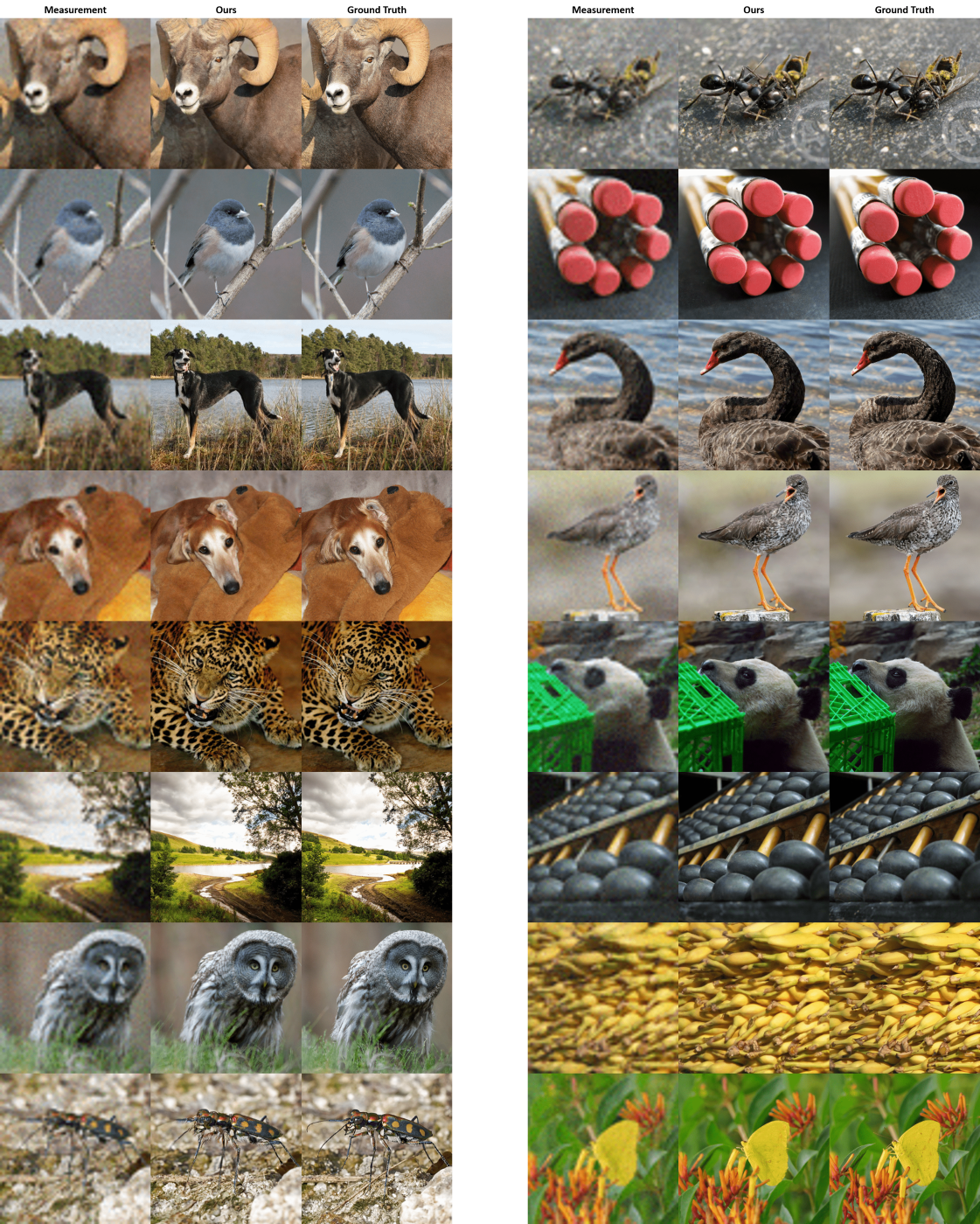}}}
  \caption{The results of SaFaRI, SR on the ImageNet 256$\times$256 dataset.
  }
  \label{fig:imnt_sr}
\end{figure*}


\begin{figure*}[t]
  \centering
    \centerline{{\includegraphics[width=0.9\linewidth]{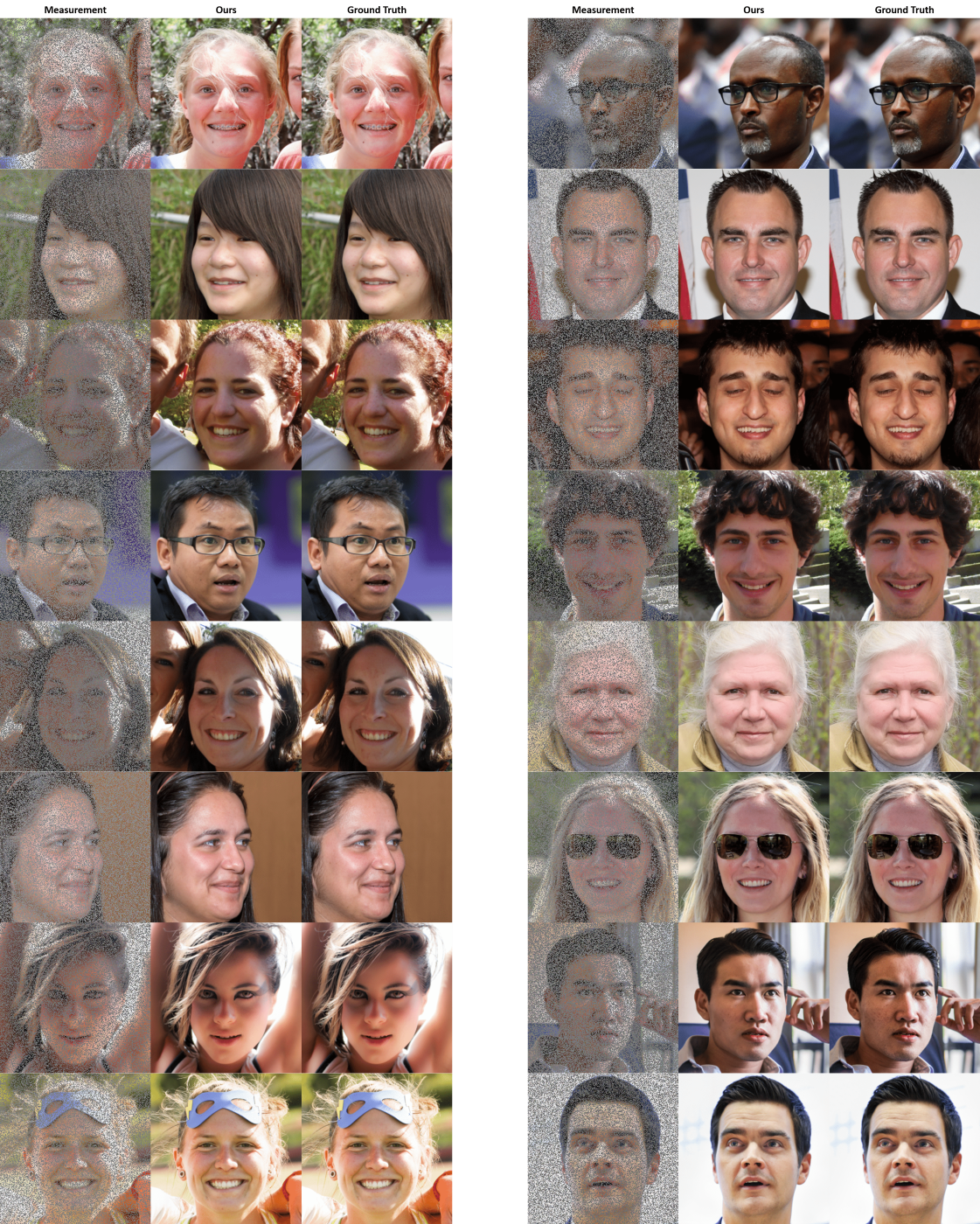}}}
  \caption{ The results of SaFaRI, random-mask inpainting on the FFHQ 256$\times$256 dataset. 
  }
  \label{fig:ffhq_ir}
\end{figure*}

\begin{figure*}[t]
  \centering
    \centerline{{\includegraphics[width=0.9\linewidth]{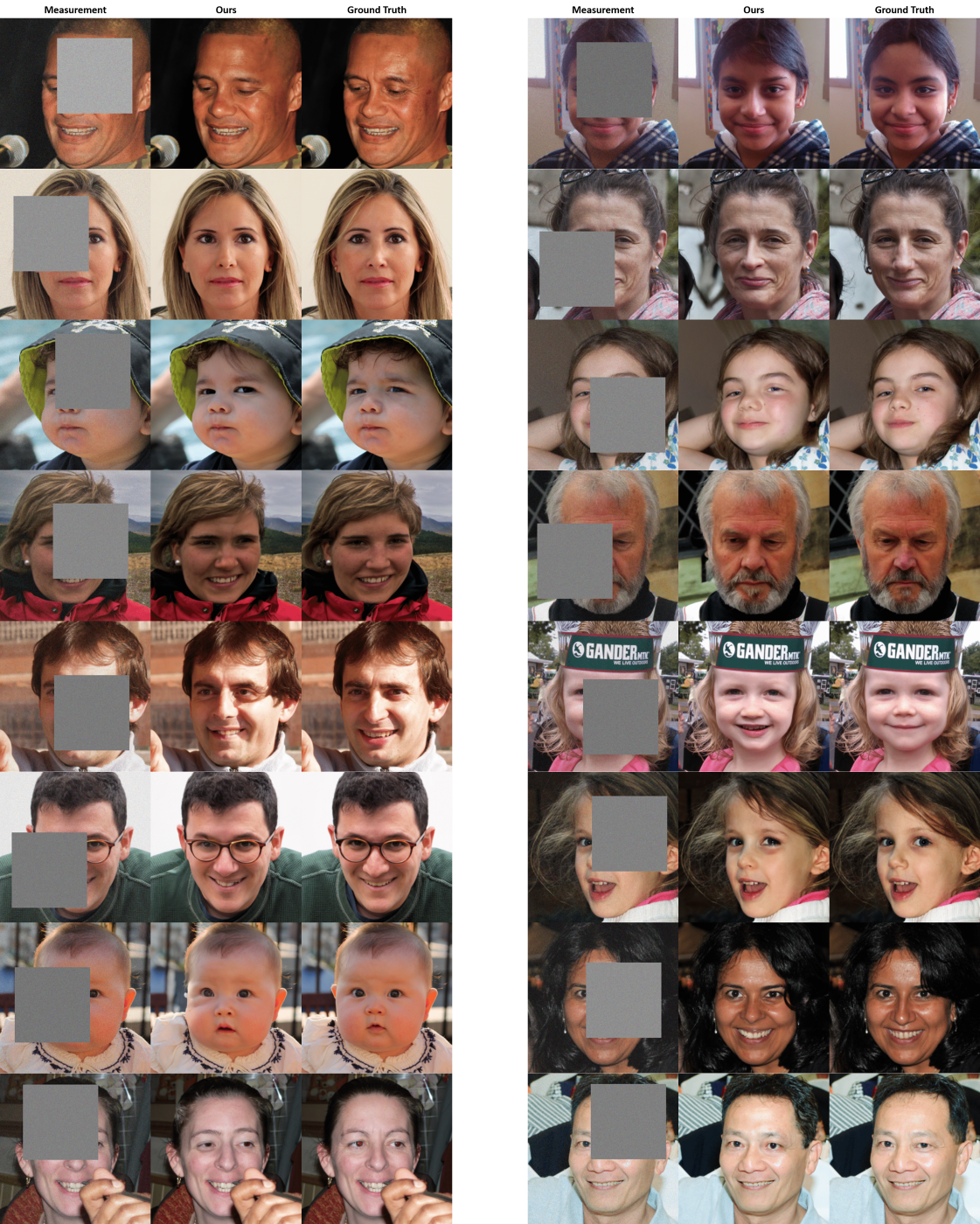}}}
  \caption{The results of SaFaRI, box-mask inpainting on the FFHQ 256$\times$256 dataset.
  }
  \label{fig:ffhq_ib}
\end{figure*}

\begin{figure*}[t]
  \centering
    \centerline{{\includegraphics[width=0.9\linewidth]{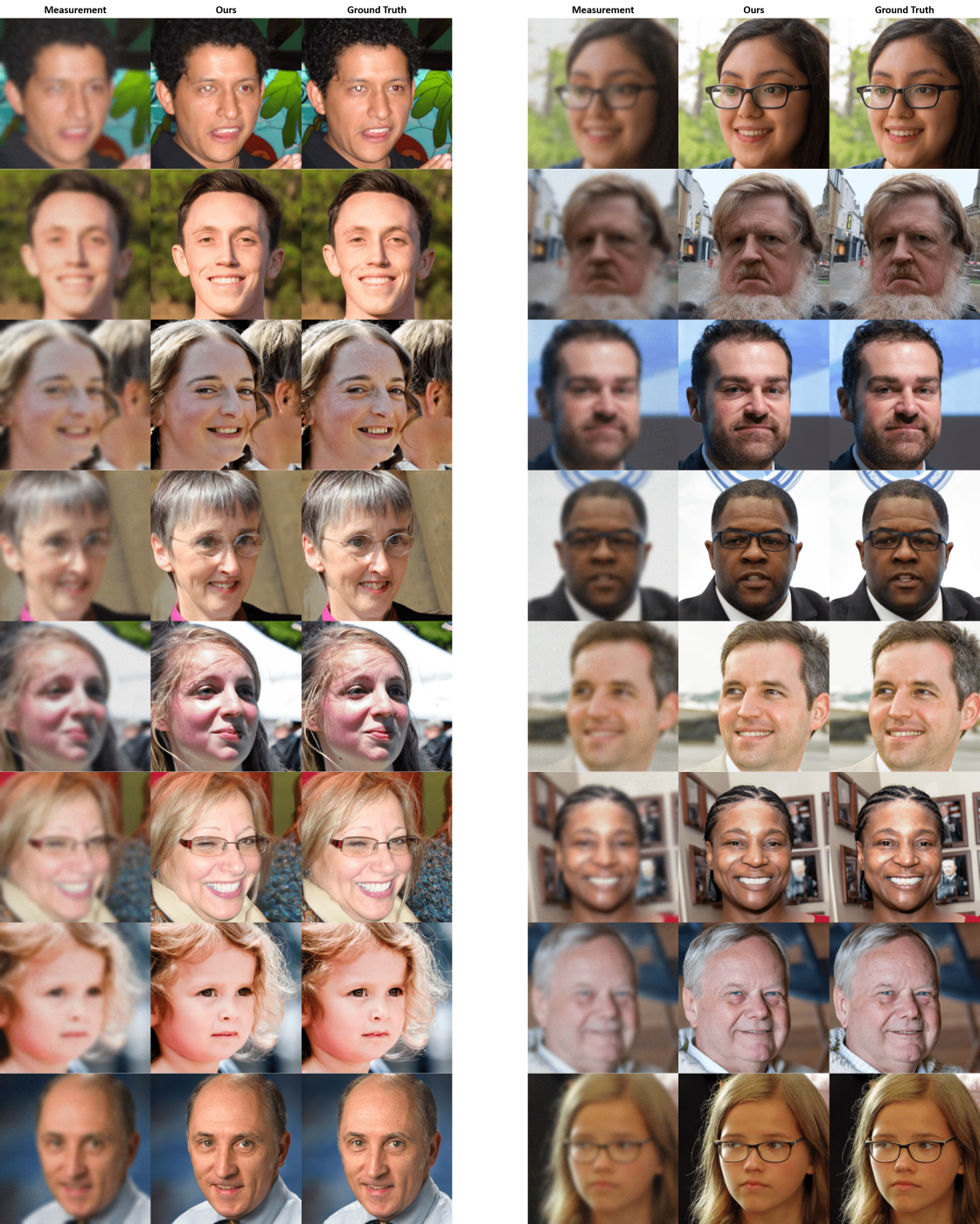}}}
  \caption{The results of SaFaRI, Gaussian deblurring on the FFHQ 256$\times$256 dataset.
  }
  \label{fig:ffhq_gd}
\end{figure*}

\begin{figure*}[t]
  \centering
    \centerline{{\includegraphics[width=0.9\linewidth]{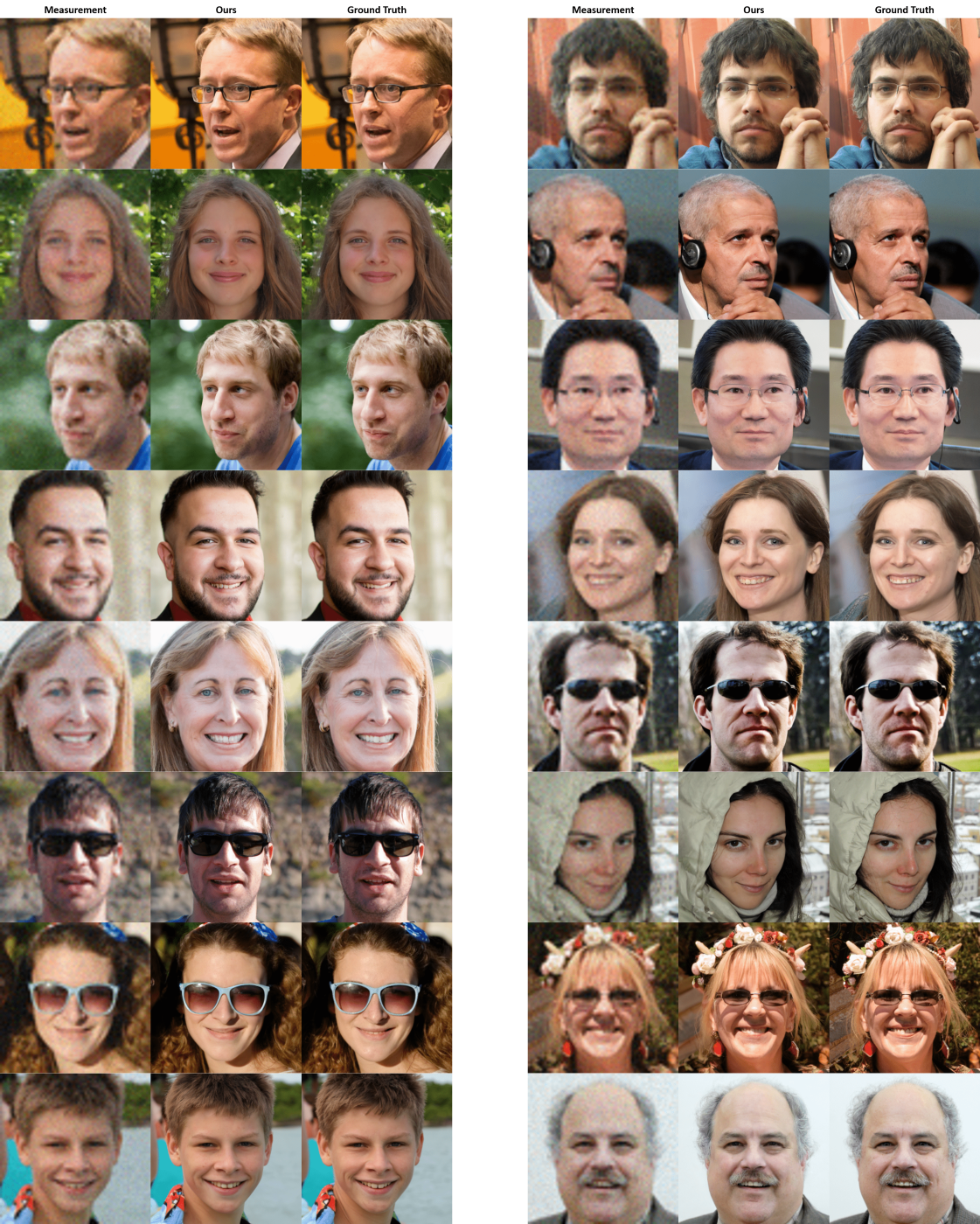}}}
  \caption{The results of SaFaRI, SR on the FFHQ 256$\times$256 dataset.
  }
  \label{fig:ffhq_sr}
\end{figure*}

\end{document}